\documentclass[letterpaper]{article}
\usepackage[top=1.5in, bottom=1.5in, left=1.20in, right=1.20in]{geometry}

\PassOptionsToPackage{numbers}{natbib}
\pdfoutput=1
\usepackage{hyperref}       % hyperlinks
\usepackage{url}            % simple URL typesetting
\usepackage{booktabs}       % professional-quality tables
\usepackage{amsfonts}       % blackboard math symbols
\usepackage{nicefrac}       % compact symbols for 1/2, etc.
\usepackage{microtype}      % microtypography
\usepackage{natbib}

\usepackage{hyperref}
\usepackage{url}
\usepackage{dsfont} % require for using \mathds{1}
\usepackage{xspace, amsmath, amssymb, amsthm}
\usepackage{color}

\pdfoutput=1

\usepackage{graphicx}
\usepackage{tabularx}

\usepackage{algorithm}
\usepackage{algorithmic}

\usepackage{multirow}% http://ctan.org/pkg/multirow
\usepackage{hhline}% http://ctan.org/pkg/hhline

\usepackage{booktabs}% to use \midrule for drawing a horizontal line in align environment

\usepackage{capt-of}
\newtheorem{lem}{Lemma}
\newtheorem{thm}{Theorem}
\newtheorem{ass}{Assumption}
\newtheorem{cor}{Corollary}
\newtheorem{defn}{Definition}

\def\eps{\ensuremath{\epsilon}\xspace}
\def\expt{\ensuremath{\mathds{E}}}

\allowdisplaybreaks % This allows breaking multi-line equations

\def\larrow{\ensuremath{\leftarrow}\xspace} 
\def\rarrow{\ensuremath{\rightarrow}\xspace} 
\def\expt{\ensuremath{\mathds{E}}}

\def\T{\ensuremath{\top}}  % transpose
\def\sig{\ensuremath{\sigma}\xspace}

\def\P{\ensuremath{\mathbb{P}}} 
\def\Pr{\ensuremath{\mbox{Pr}}} 
 
\def\eps{\ensuremath{\epsilon}\xspace}

\def\dt{{\ensuremath{\delta}\xspace} }

\newcommand{\what}[1]{ {\ensuremath{\widehat{#1}}} }

\def\barV{\ensuremath{\overline{V}}\xspace}

\def\sig{\ensuremath{\sigma}\xspace}

\usepackage[export]{adjustbox}

\def\lcl{\lceil}  
\def\rcl{\rceil}

\usepackage{pbox} % new line in a table cell -- http://tex.stackexchange.com/questions/2441/how-to-add-a-forced-line-break-inside-a-table-cell

\newcommand{\fr}[2]{ { \frac{#1}{#2} }}
\def\lt{\left}
\def\rt{\right}
\def\gam{{\ensuremath{\gamma}\xspace} }

\makeatletter
\newcommand{\vast}{\bBigg@{3}}
\newcommand{\Vast}{\bBigg@{4}}
\makeatother

\def\dsR{{{\mathds{R}}}}

\def\calX{\ensuremath{\mathcal{X}}\xspace} 
\def\cX{\ensuremath{\mathcal{X}}\xspace} 
\def\th{\ensuremath{\boldsymbol{\theta}}\xspace} 

\def\x{{{\mathbf x}}}
\def\y{{{\mathbf y}}}
\def\z{{{\mathbf z}}}
\def\la{{\langle}}
\def\ra{{\rangle}}
\def\det{\ensuremath{\mbox{det}}}
\def\lam{\ensuremath{\lambda}}
\def\barV{\ensuremath{\overline{\mathbf{V}}}\xspace} 
\def\ibarV{\ensuremath{\overline{\mathbf{V}}^{-1}}\xspace} 
\def\hatth{\ensuremath{\widehat{\boldsymbol\theta}}\xspace} 
\def\tilth{\ensuremath{\tilde{\boldsymbol\theta}}\xspace}

\def\calL{\ensuremath{\mathcal{L}}\xspace}

\usepackage{pifont}% http://ctan.org/pkg/pifont
\newcommand{\cmark}{\ding{51}}%
\newcommand{\gyes}{{\color[rgb]{0,.8,0}\cmark}}
\newcommand{\xmark}{\ding{55}}%
\newcommand{\rno}{{\color[rgb]{.8,0,0}\xmark}}

\def\p{{{\mathbf p}}}

\def\q{{{\mathbf q}}}

\def\cB{\ensuremath{\mathcal{B}}\xspace} 
 
\def\cA{\ensuremath{\mathcal{A}}\xspace}

\usepackage{enumitem}

\def\Regret{\ensuremath{\normalfont{\text{Regret}}}}

\def\ccH{\ensuremath{c_\text{\normalfont H}}}

\def\bfeta{\ensuremath{\boldsymbol{\eta}}} 
\def\X{\ensuremath{\mathbf{X}}} 
\def\I{\ensuremath{\mathbf{I}}} 
\def\til{\tilde} 
\def\A{\ensuremath{\mathbf{A}}}

\newcommand{\wbar}[1]{ {\ensuremath{\overline{#1}}} }
\def\barbeta{\wbar{\beta}}

\def\dsR{\ensuremath{{\mathds R}}\xspace}
\def\calX{\ensuremath{{\mathcal X}}\xspace}

\def\th{\ensuremath{\theta}\xspace}
\def\hth{\ensuremath{\what{\theta}}\xspace}
 
\def\th{{{\boldsymbol \theta}}} 
\def\bfeta{{{\boldsymbol \eta}}} 
 
\def\hth{{\what{\boldsymbol \theta}}} 

\def\X{\ensuremath{\mathbf{X}}\xspace}

\def\x{\ensuremath{\mathbf{x}}\xspace}

\def\lam{{\ensuremath{\lambda}\xspace} }

\def\a{\ensuremath{\mathbf{a}}}
\def\tL{\ensuremath{\mathbf{\text{L}}}}

\def\Z{\ensuremath{\mathbf{Z}}}

\def\lnorm{\left|\left|}
\def\rnorm{\right|\right|}

\usepackage{accents}
\newlength{\dhatheight}

\def\kap{{\ensuremath{\kappa}}}

\newcommand{\blue}[1]{{\color[rgb]{.1,.1,1}#1}}

\usepackage[export]{adjustbox}
\def\imagetop#1{\vtop{\null\hbox{#1}}}

\def\Var{{\ensuremath{\text{\normalfont Var}}}}

\def\ONS{{\ensuremath{\text{\normalfont ONS}}}}        % \def\SAC{{\textbf{SAC}}\xspace}   
\def\GLOC{{\ensuremath{\text{\normalfont GLOC}}}}        % \def\SAC{{\textbf{SAC}}\xspace}   
\def\QGLOC{{\ensuremath{\text{\normalfont QGLOC}}}}        % \def\SAC{{\textbf{SAC}}\xspace}   
\def\GLOCTS{{\ensuremath{\text{\normalfont GLOC-TS}}}}        % \def\SAC{{\textbf{SAC}}\xspace}   

\def\barW{\ensuremath{\overline{\mathbf{W}}}} 
\def\ibarW{\ensuremath{\overline{\mathbf{W}}^{-1}}} 

\def\vec{\ensuremath{{\text{vec}}}}
\def\bfxi{{{\boldsymbol\xi}}}
\def\cN{\ensuremath{\mathcal{N}}\xspace} 

\def\cH{\ensuremath{\mathcal{H}}\xspace}

\newcommand{\vct}[1]{\mathbf{#1}}

\newcommand{\vq}{\vct{q}}
\newcommand{\va}{\vct{a}}

\newcommand{\lv}{\left\vert}
\newcommand{\rv}{\right\vert}

\newcommand{\lV}{\left\Vert}
\newcommand{\rV}{\right\Vert}

\renewcommand{\Pr}{\mathbb{P}}

\def\dtH{\ensuremath{\dt_\text{\normalfont H}}}
\def\barm{\ensuremath{\overline{m}}}

\usepackage{wrapfig}
\usepackage{lipsum}
\usepackage{microtype}      % microtypography

\title{Scalable Generalized Linear Bandits: \\Online Computation and Hashing}

\author{
  \begin{tabular}{cc}
  Kwang-Sung Jun & Aniruddha Bhargava\\
  UW-Madison & UW-Madison \\
  \texttt{kjun@discovery.wisc.edu} & \texttt{aniruddha@wisc.edu} \\
   & \\
  Robert Nowak & Rebecca Willett\\
  UW-Madison & UW-Madison\\
  \texttt{rdnowak@wisc.edu} & \texttt{willett@discovery.wisc.edu} \\
  \end{tabular}
}

\begin{document}

\setlength{\abovedisplayskip}{5pt}
\setlength{\belowdisplayskip}{4pt}
\setlength{\abovedisplayshortskip}{5pt}
\setlength{\belowdisplayshortskip}{4pt}

%\setlength{\parindent}{0pt}
%\the\parskip
%\setlength{\parskip}{5pt}
%\renewcommand{\baselinestretch}{1.5} % this is line spacing


\date{}
\maketitle
%%%%%%%%%%%%%%%%%%%%%%%%%%%%%%%%%%%%%%%%%%%%%%%%%%%%%%%%%%%%%%%%%%%%%%%%%%%%%%%%
%%% NOTE
% 5/11 10:55am: before moving TS to the previous section

%%%%%%%%%%%%%%%%%%%%%%%%%%%%%%%%%%%%%%%%%%%%%%%%%%%%%%%%%%%%%%%%%%%%%%%%%%%%%%%%
\begin{abstract}
%%%%%%%%%%%%%%%%%%%%%%%%%%%%%%%%%%%%%%%%%%%%%%%%%%%%%%%%%%%%%%%%%%%%%%%%%%%%%%%%
  Generalized Linear Bandits (GLBs), a natural extension of the stochastic linear bandits, has been popular and successful in recent years.
  However, existing GLBs scale poorly with the number of rounds and the number of arms, limiting their utility in practice.
  This paper proposes new, scalable solutions to the GLB problem in two respects.
  First, unlike existing GLBs, whose per-time-step space and time complexity grow at least linearly with time $t$, we propose a new algorithm that performs online computations to enjoy a constant space and time complexity.
  At its heart is a novel Generalized Linear extension of the Online-to-confidence-set Conversion (GLOC method) that takes \emph{any} online learning algorithm and turns it into a GLB algorithm.
  As a special case, we apply GLOC to the online Newton step algorithm, which results in a low-regret GLB algorithm with much lower time and memory complexity than prior work.
  Second, for the case where the number $N$ of arms is very large, we propose new algorithms in which each next arm is selected via an inner product search.
  Such methods can be implemented via hashing algorithms (i.e., ``hash-amenable'') and result in a time complexity sublinear in $N$.
  While a Thompson sampling extension of GLOC is hash-amenable, its regret bound for $d$-dimensional arm sets scales with $d^{3/2}$, whereas GLOC's regret bound scales with $d$.
  Towards closing this gap, we propose a new hash-amenable algorithm whose regret bound scales with $d^{5/4}$.
  Finally, we propose a fast approximate hash-key computation (inner product) with a better accuracy than the state-of-the-art, which can be of independent interest.
  We conclude the paper with preliminary experimental results confirming the merits of our methods. 
  \vspace{-8pt}
\end{abstract}

%%%%%%%%%%%%%%%%%%%%%%%%%%%%%%%%%%%%%%%%%%%%%%%%%%%%%%%%%%%%%%%%%%%%%%%%%%%%%%%%
\section{Introduction}
\label{sec:intro}
%%%%%%%%%%%%%%%%%%%%%%%%%%%%%%%%%%%%%%%%%%%%%%%%%%%%%%%%%%%%%%%%%%%%%%%%%%%%%%%%

This paper considers the problem of making generalized linear bandits (GLBs) scalable.
In the stochastic GLB problem, a learner makes successive decisions to maximize her cumulative rewards.
Specifically, at time $t$ the learner observes a set of arms $\cX_t\subseteq\dsR^d$.
The learner then chooses an arm $\x_t\in\cX_t$ and receives a stochastic reward $y_t$ that is a noisy function of $\x_t$:
  $  y_t = \mu(\x_t^\T\th^*) + \eta_t $,
  where $\th^*\in\dsR^d$ is unknown, $\mu\hspace{-2pt}:\hspace{-2pt}\dsR\hspace{-2pt}\rarrow\hspace{-2pt}\dsR$ is a known nonlinear mapping, and $\eta_t\in\dsR$ is some zero-mean noise. 
This reward structure encompasses generalized linear models~\cite{mccullagh89generalized}; e.g., Bernoulli, Poisson, etc.

The key aspect of the bandit problem is that the learner does not know how much reward she would have received, had she chosen another arm.
The estimation on $\th^*$ is thus biased by the history of the selected arms, and one needs to mix in exploratory arm selections to avoid ruling out the optimal arm.
This is well-known as the exploration-exploitation dilemma. 
The performance of a learner is evaluated by its \emph{regret} that measures how much cumulative reward she would have gained additionally if she had known the true $\th^*$.
We provide backgrounds and formal definitions in Section~\ref{sec:prelim}.

A linear case of the problem above ($\mu(z)=z$) is called the (stochastic) linear bandit problem.
Since the first formulation of the linear bandits~\cite{auer02using}, there has been a flurry of studies on the problem~\cite{dani08stochastic,rusmevichientong10linearly,ay11improved,chu11contextual,agrawal13thompson}.
In an effort to generalize the restrictive linear rewards, \citet{filippi10parametric} propose the GLB problem and provide a low-regret algorithm, whose Thompson sampling version appears later in~\citet{abeille17linear}.
\citet{li12anunbiased} evaluates GLBs via extensive experiments where GLBs exhibit lower regrets than linear bandits for 0/1 rewards.
\citet{li17provable} achieves a smaller regret bound when the arm set $\cX_t$ is finite, though with an impractical algorithm.

\textit{However, we claim that all existing GLB algorithms~\cite{filippi10parametric,li17provable} suffer from two scalability issues that limit their practical use: (i) under a large time horizon and (ii) under a large number $N$ of arms.}

First, existing GLBs require storing all the arms and rewards appeared so far, $\{(\x_s,y_s)\}_{s=1}^t$, so the space complexity grows linearly with $t$.
Furthermore, they have to solve a batch optimization problem for the maximum likelihood estimation (MLE) at each time step $t$ whose per-time-step time complexity grows at least linearly with $t$.
While~\citet{zhang16online} provide a solution whose space and time complexity do not grow over time, they consider a specific 0/1 reward with the logistic link function, and a generic solution for GLBs is not provided.

% First, existing GLBs must store all the chosen arms and rewards so far, $\{(\x_s,y_s)\}_{s=1}^t$; the space complexity grows linearly with $t$.
% Moreover, they must solve the maximum likelihood estimation problem at each time step whose per-time-step time complexity grows at least linearly with $t$.
% While the algorithm of~\citet{zhang16online} enjoys constant space and time complexity in $t$, it is limited to a specific 0/1 reward with the logistic link function, and a generic solution for GLBs is not provided.

Second, existing GLBs have linear time complexities in $N$.
This is impractical when $N$ is very large, which is not uncommon in applications of GLBs such as online advertisements, recommendation systems, and interactive retrieval of images or documents~\cite{li10acontextual,li12anunbiased,yue12hierarchical,hofmann11contextual,konyushkova13content} where arms are items in a very large database.
Furthermore, the interactive nature of these systems requires prompt responses as users do not want to wait.
This implies that the typical linear time in $N$ is not tenable.
Towards a \emph{sublinear} time in $N$, locality sensitive hashings~\cite{indyk12approximate} or its extensions~\cite{shrivastava14asymmetric,shrivastava15improved,neyshabur15on} are good candidates as they have been successful in fast similarity search and other machine learning problems like active learning~\cite{jain10hashing}, where the search time scales with $N^{\rho}$ for some $\rho<1$ ($\rho$ is usually optimized and often ranges from 0.4 to 0.8 depending on the target search accuracy).
Leveraging hashing in GLBs, however, relies critically on the objective function used for arm selections.
The function must take a form that is readily optimized using \emph{existing} hashing algorithms.\footnote{
  Without this designation, no {\em currently known} bandit algorithm achieves a sublinear time complexity in $N$.%
}
For example, algorithms whose objective function (a function of each arm $\x\in\cX_t$) can be written as a distance or inner product between $\x$ and a query $\q$ are hash-amenable as there \emph{exist} hashing methods for such functions.

To be scalable to a large time horizon, we propose a new algorithmic framework called Generalized Linear Online-to-confidence-set Conversion (GLOC) that takes in an online learning (OL) algorithm with a low  `OL' regret bound and turns it into a GLB algorithm with a low `GLB' regret bound.
The key tool is a novel generalization of the online-to-confidence-set conversion technique used in~\cite{ay12online} (also similar to~\cite{dekel12selective,crammer13multiclass,gentile14onmultilabel,zhang16online}).
This allows us to construct a confidence set for $\th^*$, which is then used to choose an arm $\x_t$ according to the well-known optimism in the face of uncertainty principle.
By relying on an online learner, GLOC inherently performs online computations and is thus free from the scalability issues in large time steps.
While any online learner equipped with a low OL regret bound can be used, we choose the online Newton step (ONS) algorithm and prove a tight OL regret bound, which results in a practical GLB algorithm with almost the same regret bound as existing inefficient GLB algorithms.
We present our proposed algorithms and their regret bounds in Section~\ref{sec:gloc}.

\begin{wrapfigure}{R}{0.45\textwidth}
  \vspace{-20pt}
  \hspace{-7pt}
\begin{minipage}{0.45\textwidth}
\begin{table}[H]
  {\centering
  \begin{tabular}{ccc} \hline	
    Algorithm & Regret                  & Hash-amenable  \\ \hline
GLOC      & $\tilde O(d\sqrt{T})$        & \rno           \\ 
GLOC-TS   & $\tilde O(d^{3/2} \sqrt{T})$ & \gyes          \\
QGLOC     & $\tilde O(d^{5/4} \sqrt{T})$ & \gyes          \\ \hline
  \end{tabular}
  \vspace{-5pt}
    \caption{Comparison of GLBs algorithms for $d$-dimensional arm sets 
      $T$ is the time horizon.
      QGLOC achieves the smallest regret among hash-amenable algorithms.
       } 
    \label{tab:bandits}
  }
\end{table}
\end{minipage}
\vspace{-10pt}
\end{wrapfigure}
For large number $N$ of arms, our proposed algorithm GLOC is not hash-amenable, to our knowledge, due to its nonlinear criterion for arm selection.
As the first attempt, we derive a Thompson sampling~\cite{agrawal13thompson,abeille17linear} extension of GLOC (GLOC-TS), which is hash-amenable due to its linear criterion.
However, its regret bound scales with $d^{3/2}$ for $d$-dimensional arm sets, which is far from $d$ of GLOC.
Towards closing this gap, we propose a new algorithm Quadratic GLOC (QGLOC) with a regret bound that scales with $d^{5/4}$.
We summarize the comparison of our proposed GLB algorithms in Table~\ref{tab:bandits}.
In Section~\ref{sec:hashing}, we present GLOC-TS, QGLOC, and their regret bound.

Note that, while hashing achieves a time complexity sublinear in $N$, there is a nontrivial overhead of computing the projections to determine the hash keys.
As an extra contribution, we reduce this overhead by proposing a new sampling-based approximate inner product method.
Our proposed sampling method has smaller variance than the state-of-the-art sampling method proposed by~\cite{jain10hashing,kannan09spectral} when the vectors are normally distributed, which fits our setting where projection vectors are indeed normally distributed. 
Moreover, our method results in thinner tails in the distribution of estimation error than the existing method, which implies a better concentration.
We elaborate more on reducing the computational complexity of QOFUL in Section~\ref{sec:iprod}.

%%%%%%%%%%%%%%%%%%%%%%%%%%%%%%%%%%%%%%%%%%%%%%%%%%%%%%%%%%%%%%%%%%%%%%%%%%%%%%%%
\vspace{-4pt}
\section{Preliminaries}
\label{sec:prelim}
\vspace{-4pt}
%%%%%%%%%%%%%%%%%%%%%%%%%%%%%%%%%%%%%%%%%%%%%%%%%%%%%%%%%%%%%%%%%%%%%%%%%%%%%%%%

We review relevant backgrounds here.  
$\cA$ refers to a GLB algorithm, and $\cB$ refers to an online learning algorithm.
Let $\cB_d(S)$ be the $d$-dimensional Euclidean ball of radius $S$, which overloads the notation $\cB$. % but is clear from $\cB$ from the context.
Let $\A_{\cdot i}$ be the $i$-th column vector of a matrix $\A$.
Define $||\x ||_\A := \sqrt{\x^\T\A\x}$ and $\vec(\A) := [\A_{\cdot 1}; \A_{\cdot 2}; \cdots ; \A_{\cdot d} ] \in \dsR^{d^2}$
Given a function $f:\dsR\rarrow\dsR$, we denote by $f'$ and $f''$ its first and second derivative, respectively.
We define $[N] := \{1,2,\ldots,N\}$.

%%%%%%%%%%%%%%%%%%%%
\vspace{-4pt}
\paragraph{Generalized Linear Model (GLM)}
%%%%%%%%%%%%%%%%%%%%

Consider modeling the reward $y$ as one-dimensional exponential family such as Bernoulli or Poisson.
When the feature vector $\x$ is believed to correlate with $y$, one popular modeling assumption is the generalized linear model (GLM) that turns the \emph{natural parameter} of an exponential family model into $\x^\T\th^*$ where $\th^*$ is a parameter~\cite{mccullagh89generalized}:
%%% BEG
% When we have a feature representation $\x$ of the arm that we believe correlates with the reward $y$, it is natural to consider a linear parameterization of the exponential family.
% Specifically, the generalized linear model (GLM) turns the \emph{natural parameter} of an exponential family into $\x^\T\th^*$ where $\th^*$ is a parameter~\cite{mccullagh89generalized}:
%%%
%\vspace{-4pt}
\begin{equation}\begin{aligned}\label{eq:glm}
  \P(y \mid z=\x^\T\th^*) = \exp\lt( \fr{y z - m(z)}{g(\tau)} + h(y,\tau) \rt) \;,
\end{aligned}\end{equation}
where $\tau \in \dsR^+$ is a known scale parameter and $m$, $g$, and $h$ are normalizers.
It is known that $m'(z) = \expt[ y \mid z] =: \mu(z)$ and $m''(z) = \Var(y\mid z)$.
We call $\mu(z)$ the \emph{inverse link} function.
Throughout, we assume that the exponential family being used in a GLM has a \emph{minimal representation}, which ensures that $m(z)$ is strictly convex~\cite[Prop. 3.1]{wainwright08graphical}.
Then, the negative log likelihood (NLL) $\ell(z,y) := -yz + m(z)$ of a GLM is strictly convex.
We refer to such GLMs as the \emph{canonical} GLM.
%%% BEG
%Note that the NLL is proportional to $-y \cdot (\x^\T\th^*) + \mu(\x^\T\th^*) $.
%%% END
In the case of Bernoulli rewards $y \in \{0,1\}$, $m(z) = \log(1+\exp(z))$, $\mu(z) = (1+\exp(-z))^{-1}$, and the NLL can be written as the logistic loss: $\log(1 + \exp(-y'(\x_t^\T\th^*))) $, where $y' = 2y - 1$.

%%%%%%%%%%%%%%%%%%%%
\vspace{-4pt}
\paragraph{Generalized Linear Bandits (GLB)}
%%%%%%%%%%%%%%%%%%%%

Recall that $\x_t$ is the arm chosen at time $t$ by an algorithm.
We assume that the arm set $\cX_t$ can be of an infinite cardinality, although we focus on finite arm sets in hashing part of the paper (Section~\ref{sec:hashing}).
One can write down the reward model~\eqref{eq:glm} in a different form:
\vspace{-4pt}
\begin{equation}\begin{aligned}\label{eq:reward}
   y_t = \mu( \x_t^\T\th^* ) + \eta_t,
\end{aligned}\end{equation}
where $\eta_t$ is conditionally $R$-sub-Gaussian given $\x_t$ and $\{(\x_s,\eta_s)\}_{s=1}^{t-1}$.
For example, Bernoulli reward model has $\eta_t$ as $1-\mu(\x_t^\T\th^*)$ w.p. $\mu(\x_t^\T\th^*)$ and $-\mu(\x_t^\T\th^*)$ otherwise.
Assume that $||\th^*||_2 \le S$, where $S$ is known. 
One can show that the sub-Gaussian scale $R$ is determined by $\mu$: $R = \sup_{z\in(-S,S)} \sqrt{\mu'(z)} \le \sqrt{L}$, where $L$ is the Lipschitz constant of $\mu$.
Throughout, we assume that each arm has $\ell_2$-norm at most 1: $||\x||_2 \le 1, \forall \x\in \cX_t, \forall t$. 
Let $\x_{t,*} := \max_{\x\in\cX_t} \x^\T\th^*$.
The performance of a GLB algorithm $\cA$ is analyzed by the expected cumulative regret (or simply \emph{regret}): $\text{Regret}^\cA_T := \sum_{t=1}^T \mu( \x_{t,*}^\T\th^* ) - \mu( (\x^\cA_t)^\T\th^* ) $, where $\x^\cA_t$ makes the dependence on $\cA$ explicit.

We remark that our results in this paper hold true for a strictly larger family of distributions than the canonical GLM, which we call the \emph{non-canonical} GLM and explain below.
The condition is that the reward model follows~\eqref{eq:reward} where the $R$ is now independent from $\mu$ that satisfies the following:
\begin{ass} \label{ass:mu}
  $\mu$ is $L$-Lipschitz on $[-S,S]$ and continuously differentiable on $(-S,S)$.
  Furthermore, $\inf_{z \in (-S,S)} \mu'(z) = \kappa$ for some finite $\kappa>0$ (thus $\mu$ is strictly increasing).
\end{ass}
%\vspace{-7pt}
Define $\mu'(z)$ at $\pm S$ as their limits.
Under Assumption~\ref{ass:mu}, $m$ is defined to be an integral of $\mu$.
Then, one can show that $m$ is $\kap$-strongly convex on $\cB_1(S)$.
An example of the non-canonical GLM is the probit model for 0/1 reward where $\mu$ is the Gaussian CDF, which is popular and competitive to the Bernoulli GLM as evaluated by~\citet{li12anunbiased}.
% We remark that~\citet{zhang16online} is specialized to the Bernoulli GLM and thus not applicable to the probit model.
Note that canonical GLMs satisfy Assumption~\ref{ass:mu}.

%%%%%%%%%%%%%%%%%%%%%%%%%%%%%%%%%%%%%%%%%%%%%%%%%%%%%%%%%%%%%%%%%%%%%%%%%%%%%%%%
\vspace{-4pt}
\section{Generalized Linear Bandits with Online Computation}
\label{sec:gloc}
\vspace{-4pt}
%%%%%%%%%%%%%%%%%%%%%%%%%%%%%%%%%%%%%%%%%%%%%%%%%%%%%%%%%%%%%%%%%%%%%%%%%%%%%%%%

We describe and analyze a new GLB algorithm called Generalized Linear Online-to-confidence-set Conversion (GLOC) that performs online computations, unlike existing GLB algorithms.

GLOC employs the optimism in the face of uncertainty principle, which dates back to~\cite{auer02using}.
That is, we maintain a confidence set $C_t$ (defined below) that traps the true parameter $\th^*$ with high probability (w.h.p.) and choose the arm with the largest feasible reward given $C_{t-1}$ as a constraint:
\begin{equation}\begin{aligned} \label{eq:glocopt}
          (\x_t,\tilth_t) := \arg \max_{\x\in\cX_t, \th\in C_{t-1}} \la\x,\th \ra
\end{aligned}\end{equation}
The main difference between GLOC and existing GLBs is in the computation of the $C_t$'s. Prior methods involve ``batch" computations that involve all past observations, and so scale poorly with $t$. In contrast, GLOC takes in an \emph{online} learner $\cB$, and uses $\cB$ as a co-routine instead of relying on a batch procedure to construct a confidence set.
Specifically, at each time $t$ GLOC feeds the loss function $\ell_t(\th) := \ell(\x_t^\T\th, y_t)$ into the learner $\cB$ which then outputs its parameter prediction $\th_t$.
Let $\X_{t} \in \dsR^{t\times d}$ be the design matrix consisting of $\x_1,\ldots,\x_t$. 
Define $\barV_t := \lam \I + \X_{t}^\T\X_{t} $, where $\lam$ is the ridge parameter.
Let $z_t := \x_t^\T\th_t$ and  $\z_t := [z_1;\cdots;z_t]$.
Let $\hth_t := \ibarV_t \X_{t}^\T \z_{t}$ be the ridge regression estimator taking $\z_t$ as responses.
Theorem~\ref{thm:o2cs} below is the key result for constructing our confidence set $C_t$, which is a function of the parameter predictions $\{\th_s\}_{s=1}^t$ and the online (OL) regret bound $B_t$ of the learner $\cB$.
All the proofs are in the supplementary material (SM).
\begin{thm} \label{thm:o2cs} (Generalized Linear Online-to-Confidence-Set Conversion)
  Suppose we feed loss functions $\{\ell_s(\th)\}_{s=1}^t$ into online learner $\cB$.
  Let $\th_s$ be the parameter predicted at time step $s$ by $\cB$.
  Assume that $\cB$ has an OL regret bound $B_t$: $\forall \th\in\cB_{d}(S), \forall t \ge 1,  $ 
  \begin{equation}\begin{aligned} \label{eq:ol_regret}
    \textstyle\sum_{s=1}^t \ell_s(\th_s) - \ell_s(\th) \le \blue{B_t} \;. 
  \end{aligned}\vspace{2pt}\end{equation}
  Let $\alpha(B_t) := 1 + \fr{4}{\kappa}\blue{B_t} + \fr{8R^2}{\kappa^2} \log(\fr{2}{\dt}\sqrt{ 1+ \fr{2}{\kap}\blue{B_t} + \fr{4R^4}{\kappa^4\dt^2} } )$.
  Then, with probability (w.p.) at least $1-\dt$,
  \begin{equation}\label{eq:thm_o2cs}
    \forall t\ge1, ||\th^* - \hth_t||^2_{\barV_t}  \le \alpha(\blue{B_t}) + \lam S^2 - \lt(||\z_{t}||^2_2 - \hth_t^\T\X_{t}^\T\z_{t}  \rt) =: \beta_t \;.
  \end{equation}
\end{thm}
Note that the center of the ellipsoid is the ridge regression estimator on the predicted natural parameters $z_s = \x_s^\T\th_s$ rather than the rewards.
Theorem~\ref{thm:o2cs} motivates the following confidence set:
\begin{equation}\begin{aligned} \label{eq:cset}
  C_t := \{ \th\in\dsR^d: ||\th - \hth_{t}||^2_{\barV_t} \le \beta_{t} \}
\end{aligned}\end{equation}
which traps $\th^*$ for all $t\ge1$, w.p. at least $1-\dt$.
%We describe the pseudocode of GLOC in Algorithm~\ref{alg:gloc}.
See Algorithm~\ref{alg:gloc} for pseudocode.
%%% BEG
% Note that~\eqref{eq:thm_o2cs} can be directly used to construct a confidence set by replacing $\th^*$ with a generic $\th$ since $\th^* \in \{\th: \sum_{s=1}^t (\x_s^\T(\th_s - \th))^2 \le \beta'_t\}$.
% However, the quadratic equation on $\th$ on the LHS of~\eqref{eq:thm_o2cs} could have a zero eigenvalue in its Hessian, which is cumbersome.
% Corollary~\ref{cor:o2cs} below presents a confidence set that takes an explicit ellipsoidal form, which is easier to work with.
% \begin{cor}\label{cor:o2cs}
%   Consider the same assumptions as Theorem~\ref{thm:o2cs}.
%   Let $\X_{t} \in \dsR^{t\times d}$ be the design matrix consisting of $\x_1,\ldots,\x_t$.
%   Define $z_s := \x_s^\T\th_s$,  $\z_{t} = (z_1,\ldots,z_t)^\T$, $\barV_t := \lam \I + \X_{t}^\T\X_{t} $ for some $\lam >0$, and $\hth_t := \ibarV_t \X_{t}^\T \z_{t}$.
%   Then, w.p. at least $1-\dt$, $\forall t\ge1$,
%   \vspace{-5pt}
%   \begin{equation*}
%      ||\th^* - \hth_t||^2_{\barV_t}  \le \beta'_t + \lam S^2 - \lt(||\z_{t}||^2_2 - \hth_t^\T\X_{t}^\T\z_{t}  \rt) =: \beta_t \;.
%   \end{equation*}
%   Furthermore, w.p. at least $1-\dt$, $\th^* \in C_t := \{ \th\in\dsR^d: ||\th - \hth_{t}||^2_{\barV_t} \le \beta_{t} \}, \forall t\ge1$.
% \end{cor}
% \vspace{-7pt}
%%% END
%We remark that one way to solve the optimization problem~\eqref{eq:glocopt} is to fix $\x$ and find the maximizer $\th(\x) := \max_{\th\in C_{t-1}} \x^\T\th$ in a closed form using the Lagrangian method:
One way to solve the optimization problem~\eqref{eq:glocopt} is to define the function $\th(\x) := \max_{\th\in C_{t-1}} \x^\T\th$, and then use the Lagrangian method to write:
\begin{equation}\begin{aligned}\label{eq:glocopt_ext}
  \x^{\GLOC}_t := \arg \max_{\x\in\cX_t} \x^\T\hth_{t-1} + \sqrt{\beta_{t-1}} ||\x||_{\ibarV_{t-1}} \;.
\end{aligned}\end{equation}
We prove the regret bound of GLOC in the following theorem.
\begin{thm}  \label{thm:regret_o2cs}
  Let $\{\barbeta_t\}$ be a nondecreasing sequence such that $\barbeta_t\ge\beta_t$. Then, w.p. at least $1-\dt$,
  \[
    \textstyle    \Regret^{\emph\GLOC}_T = O\lt( L\sqrt{\barbeta_T dT\log T} \rt)
  \]
\end{thm}
\begin{wrapfigure}{R}{0.48\textwidth}
  \vspace{-22pt}
%  \hspace{-7pt}
\begin{minipage}{0.48\textwidth}
\begin{algorithm}[H]
{\small
  \begin{algorithmic}[1]
    \STATE \textbf{Input}: $R>0$, $\dt\in(0,1)$, $S>0$, $\lam>0$, $\kappa>0$, an online learner $\cB$ with known regret bounds $\{B_t\}_{t\ge1}$. 
    \STATE Set $\barV_0 = \lam \I$.
    \FOR {$t=1,2,\ldots$}
    \STATE Compute $\x_t$ by solving~\eqref{eq:glocopt}.
      \STATE Pull $\x_t$ and then observe $y_t$.
      \STATE Receive $\th_t$ from $\cB$.
      \STATE Feed into $\cB$ the loss $\ell_t(\th) = \ell(\x_t^\T \th, y_t)$.
      \STATE Update $\barV_t = \barV_{t-1} + \x_t\x_t^\T$ and $z_t = \x_t^\T\th_t$
      \STATE Compute $\hth_t = \ibarV_t \X_t^\T\z_t$  and $\beta_t$ as in~\eqref{eq:thm_o2cs}.
      \STATE Define $C_t$ as in~\eqref{eq:cset}.
    \ENDFOR
  \end{algorithmic}
  \caption{GLOC}
  \label{alg:gloc}
}
\end{algorithm}
\vspace{-20pt}
\begin{algorithm}[H]
{\small
  \begin{algorithmic}[1]
    \STATE \textbf{Input}: $\kap>0$, $\eps>0$, $S >0$.%$\Th\subseteq\dsR^d$
    \STATE $\A_0 = \eps\I$.
    \STATE Set $\th_1 \in \cB_d(S)$ arbitrarily.
    \FOR {$t=1,2,3,\ldots$}
      \STATE Output $\th_t$ .
      \STATE Observe $\x_t$ and $y_t$. 
      \STATE Incur loss $\ell(\x_t^\T\th_t, y_t)$ .
      \STATE $\A_t = \A_{t-1} + \x_t \x_t^\T$
      \STATE $\th'_{t+1} = \th_{t} - \fr{\ell'(\x_{t}^\T\th_{t}, y_t)}{\kap} \A^{-1}_{t} \x_{t} $
      \STATE $\th_{t+1} = \arg \min_{\th\in\cB_d(S)} || \th - \th'_{t+1} ||^2_{\A_{t}}  $   
    \ENDFOR
  \end{algorithmic}
  \caption{ONS-GLM}
  \label{alg:ons}
}
\end{algorithm}
\end{minipage}
\vspace{-15pt}
\end{wrapfigure}
Although any low-regret online learner can be combined with GLOC, one would like to ensure that $\barbeta_T$ is $O(\text{polylog}(T))$ in which case the total regret can be bounded by $\tilde O(\sqrt{T})$.
This means that we must use online learners whose OL regret grows logarithmically in $T$ such as~\cite{hazan07logarithmic,orabona12beyond}.
In this work, we consider the online Newton step (ONS) algorithm~\cite{hazan07logarithmic}.

%%%%%%%%%%%%%%%%%%%%
\vspace{-4pt}
\paragraph{Online Newton Step (ONS) for Generalized Linear Models}
%%%%%%%%%%%%%%%%%%%%

% \begin{wrapfigure}{R}{0.46\textwidth}
%   \vspace{-25pt}
%   \hspace{-8pt}
% \begin{minipage}{0.46\textwidth}
% \begin{algorithm}[H]
%   \begin{algorithmic}[1]
%     \STATE \textbf{Input}: $\kap>0$, $\eps>0$, $S >0$.%$\Th\subseteq\dsR^d$
%     \STATE $\A_0 = \eps\I$.
%     \STATE Set $\th_1 \in \cB_d(S)$ arbitrarily.
% %    \STATE Choose $\th_1 \in \Th$ arbitrarily.
%     \FOR {$t=1,2,3,\ldots$}
%       \STATE Output $\th_t$ .
%       \STATE Observe $\x_t$ and $y_t$. 
%       \STATE Incur loss $\ell(\x_t^\T\th_t, y_t)$ .
%       \STATE $\A_t = \A_{t-1} + \x_t \x_t^\T$
%       \STATE $\th'_{t+1} = \th_{t} - \fr{\ell'(\x_{t}^\T\th_{t}, y_t)}{\kap} \A^{-1}_{t} \x_{t} $
%       \STATE $\th_{t+1} = \arg \min_{\th\in\cB_d(S)} || \th - \th'_{t+1} ||^2_{\A_{t}}  $   
%     \ENDFOR
%   \end{algorithmic}
%   \caption{Online Newton Step for Generalized Linear Model (ONS-GLM)}
%   \label{alg:ons}
% \end{algorithm}
% \end{minipage}
% \vspace{-15pt}
% \end{wrapfigure}

Note that ONS requires the loss functions to be $\alpha$-exp-concave.
One can show that $\ell_t(\th)$ is $\alpha$-exp-concave~\cite[Sec. 2.2]{hazan07logarithmic}.
Then, GLOC can use ONS and its OL regret bound to solve the GLB problem. %are immediately ready to be combined with GLOC.
However, motivated by the fact that the OL regret bound $B_t$ appears in the radius $\sqrt{\beta_t}$ of the confidence set while a tighter confidence set tends to reduce the bandit regret in practice, we derive a tight data-dependent OL regret bound tailored to GLMs.

We present our version of ONS for GLMs (ONS-GLM) in Algorithm~\ref{alg:ons}.
$\ell'(z,y)$ is the first derivative w.r.t. $z$ and the parameter $\eps$ is for inverting matrices conveniently (usually $\eps = $ 1 or 0.1).
The only difference from the original ONS~\cite{hazan07logarithmic} is that we rely on the strong convexity of $m(z)$ instead of the $\alpha$-exp-concavity of the loss thanks to the GLM structure.\footnote{ A similar change to ONS has been applied in~\cite{gentile14onmultilabel,zhang16online}.}
Theorem~\ref{thm:ons} states that we achieve the desired polylogarithmic regret in $T$.
\begin{thm}\label{thm:ons}
  Define $g_s := \ell'(\x_s^\T\th_s, y_s)$.
  The regret of ONS-GLM satisfies, for any $\eps > 0$ and $t\ge 1$,
  \[
    \textstyle \sum_{s=1}^t \ell_s( \th_s) - \ell_s(\th^*) \le \fr{1}{2\kap} \sum_{s=1}^t g_s^2 ||\x_s||^2_{\A_s^{-1}} + 2\kap S^2 \eps =: B^{\emph\ONS}_t \;,
  \]
  where $B^\ONS_t = O(\fr{L^2+ R^2\log(t)}{\kap}d\log t), \forall t\ge1$ w.h.p. If $\max_{s\ge1} |\eta_s|$ is bounded by $\bar R$ w.p. 1, $B^\ONS_t = O(\fr{L^2+\bar R^2}{\kap}d\log t)$.
\end{thm}
We emphasize that the OL regret bound is data-dependent. % and its magnitude can be bounded by problem constants (shown after the proof in SM).
A confidence set constructed by combining Theorem~\ref{thm:o2cs} and Theorem~\ref{thm:ons} directly implies the following regret bound of GLOC with ONS-GLM.
\begin{cor} \label{cor:cset_ONS}
  Define $\beta_t^{\emph\ONS}$ by replacing $B_t$ with $B^{\emph\ONS}_t$ in~\eqref{eq:thm_o2cs}.
  With probability at least $1-2\dt$, 
  \begin{equation}\begin{aligned} \label{eq:cset_ONS}
      \forall t\ge1, \th^* \in C^{\emph\ONS}_t := \lt\{\th\in\dsR^d : || \th - \hth_t ||^2_{\barV_t} \le \beta^{\emph\ONS}_t \rt\} \;.
  \end{aligned}\end{equation}                                                                                 
%   where $\beta^{\emph\ONS}_t = \hat O( ( \fr{L^2 + R^2}{\kap^2}) d \log^2(t) )$ and $\hat{O}$ ignores $\log\log(t)$. 
%   If $|\eta_s|$ is bounded by $\bar R$, $\beta^{\emph\ONS}_t = O( ( \fr{L^2 + {\bar R}^2}{\kap^2}) d \log(t) )$.
\end{cor}
\begin{cor} \label{cor:regret_glocon_ons}
  Run GLOC with $C^{\ONS}_t$. Then, w.p. at least $1-2\dt$, $\forall T\ge1$,
  $\Regret_T^{\GLOC} = \hat O\lt(\fr{L(L+R)}{\kap} d \sqrt{T} \log^{3/2}(T)\rt) $ where $\hat{O}$ ignores $\log\log(t)$.
  If $|\eta_t|$ is bounded by $\bar R$, $\Regret_T^{\GLOC} = \hat O\lt(\fr{L(L + \bar R)}{\kap} d \sqrt{T} \log(T)\rt)$.
\end{cor}

We make regret bound comparisons ignoring $\log\log T$ factors.
For generic arm sets, our dependence on $d$ is optimal for linear rewards~\cite{rusmevichientong10linearly}.
For the Bernoulli GLM, our regret has the same order as~\citet{zhang16online}.
One can show that the regret of~\citet{filippi10parametric} has the same order as ours if we use their assumption that the reward $y_t$ is bounded by $R_{\max}$. % (e.g., $\bar R=1/2$ for Bernoulli).
For unbounded noise,~\citet{li17provable} have regret $O((LR/\kap) d\sqrt{T} \log T)$, which is $\sqrt{\log T}$ factor smaller than ours and has $LR$ in place of $L(L+R)$.
While $L(L+R)$ could be an artifact of our analysis, the gap is not too large for canonical GLMs.
Let $L$ be the smallest Lipschitz constant of $\mu$.
Then, $R=\sqrt{L}$.
If $L \le 1$, $R$ satisfies $R > L$, and so $L(L+R)=O(LR)$.
If $L > 1$, then $L(L+R) = O(L^2)$, which is larger than $LR = O(L^{3/2})$.
For the Gaussian GLM with known variance $\sig^2$, $L=R=1$.\footnote{
  The reason why $R$ is not $\sig$ here is that the sufficient statistic of the GLM is $y/\sig$, which is equivalent to dealing with the normalized reward.
  Then, $\sig$ appears as a factor in the regret bound.
}
For finite arm sets, SupCB-GLM of~\citet{li17provable} achieves regret of $\tilde O(\sqrt{dT\log N})$ that has a better scaling with $d$ but is not a practical algorithm as it wastes a large number of arm pulls.
Finally, we remark that none of the existing GLB algorithms are scalable to large $T$.
\citet{zhang16online} is scalable to large $T$, but is restricted to the Bernoulli GLM; e.g., theirs does not allow the probit model (non-canonical GLM) that is popular and shown to be competitive to the Bernoulli GLM~\cite{li12anunbiased}.

%%%%%%%%%%%%%%%%%%%%
\vspace{-4pt}
\paragraph{Discussion} 
%%%%%%%%%%%%%%%%%%%%
%\kw{place the following in the s/m}
The trick of obtaining a confidence set from an online learner appeared first in~\cite{dekel10robust,dekel12selective} for the linear model, and then was used in~\cite{crammer13multiclass,gentile14onmultilabel,zhang16online}.
GLOC is slightly different from these studies and rather close to~\citet{ay12online} in that the confidence set is a function of a known regret bound. %, which has a greater generality.
This generality frees us from re-deriving a confidence set for every online learner. %; one reduction for all.
Our result is essentially a nontrivial extension of~\citet{ay12online} to GLMs.

One might have notice that $C_t$ does not use $\th_{t+1}$ that is available before pulling $\x_{t+1}$ and has the most up-to-date information.
This is inherent to GLOC as it relies on the OL regret bound directly.
One can modify the proof of ONS-GLM to have a tighter confidence set $C_t$ that uses $\th_{t+1}$ as we show in SM Section~\ref{sec:supp_tighter}.
However, this is now specific to ONS-GLM, which looses generality. 
%%% BEG
% One might have noticed that the confidence set $C_{t-1}$ used for choosing the arm $\x_t$ does not depend directly on $y_{t-1}$ but only through $g_{t-1}$.
% This comes from the fact that we directly use the regret bound rather than performing a fresh analysis of the online learner. 
% While this approach provides a great generality, one can indeed use $y_{t-1}$ for the confidence set $C_{t-1}$ after a careful analysis as we show in SM Section~\ref{sec:supp_tighter}.
% However, the result on a tighter confidence set is now specific to ONS, which looses generality.
%%% END

%%%%%%%%%%%%%%%%%%%%%%%%%%%%%%%%%%%%%%%%%%%%%%%%%%%%%%%%%%%%%%%%%%%%%%%%%%%%%%%%
\vspace{-04pt}
\section{Hash-Amenable Generalized Linear Bandits}
\label{sec:hashing}
\vspace{-4pt}
%%%%%%%%%%%%%%%%%%%%%%%%%%%%%%%%%%%%%%%%%%%%%%%%%%%%%%%%%%%%%%%%%%%%%%%%%%%%%%%%

We now turn to a setting where the arm set is finite but very large.
For example, imagine an interactive retrieval scenario~\cite{rui98relevance,konyushkova13content,glowacka15balancing} where a user is shown $K$ images (e.g., shoes) at a time and provides relevance feedback (e.g., yes/no or 5-star rating) on each image, which is repeated until the user is satisfied.
In this paper, we focus on showing one image (i.e., arm) at a time.\footnote{
  One image at a time is a simplification of the practical setting.
  One can extend it to showing multiple images at a time, which is a special case of the combinatorial bandits of~\citet{qin14contextual}.
}
Most existing algorithms require maximizing an objective function (e.g., \eqref{eq:glocopt_ext}), the complexity of which scales linearly with the number $N$ of arms. 
This can easily become prohibitive for large numbers of images.
Furthermore, the system has to perform real-time computations to promptly choose which image to show the user in the next round.
Thus, it is critical for a practical system to have a time complexity sublinear in $N$.

One naive approach is to select a subset of arms ahead of time, such as volumetric spanners~\cite{hazan16volumetric}.
However, this is specialized for an efficient exploration only and can rule out a large number of good arms.
Another option is to use hashing methods.
Locality-sensitive hashing and Maximum Inner Product Search (MIPS) are effective and well-understood tools but can only be used when the objective function is a distance or an inner product computation;~\eqref{eq:glocopt_ext} cannot be written in this form.
In this section, we consider alternatives to GLOC which are compatible with hashing.

%%% BEG
% Another option is to try to use a locality-sensitive hashing or Maximum Inner Product Search (MIPS) hashing to solve~\eqref{eq:glocopt_ext}. 
% However, it is unclear how to turn the objective function in~\eqref{eq:glocopt_ext} into a distance computation or an inner product since the second term $||\x||_{\ibarV_{t-1}}$ therein involves a square root that makes the function neither linear nor quadratic.
% An alternative is to use GLOC-TS whose objective function is an inner product, but its regret bound scales with $d^{3/2}$ rather than $d$ of GLOC.
% This is concerning in the interactive retrieval or product recommendation scenario since the relevance of the shown items is harmed, which makes us wonder if one can improve the regret without loosing the hash-amenability.
%%% END

%%%%%%%%%%%%%%%%%%%%
\vspace{-6pt}
\paragraph{Thompson Sampling} 
%%%%%%%%%%%%%%%%%%%%
%
We present a Thompson sampling (TS) version of GLOC called GLOC-TS that chooses an arm $\x_t = \arg \max_{\x\in\cX_t} \x^\T \dot\th_t$ where $\dot\th_t \sim \cN(\hth_{t-1}, \beta_{t-1}\ibarV_{t-1})$.
TS is known to perform well in practice~\cite{chapelle11anempirical} and can solve the polytope arm set case in polynomial time\footnote{ConfidenceBall$_1$ algorithm of~\citet{dani08stochastic} can solve the problem in polynomial time as well.} whereas algorithms that solve an objective function like~\eqref{eq:glocopt} (e.g.,~\cite{ay11improved}) cannot since they have to solve an NP-hard problem~\cite{agrawal13thompson}.
We present the regret bound of GLOC-TS below. 
Due to space constraints, we present the pseudocode and the full version of the result in SM.
\begin{thm} (Informal) \label{thm:gloc-ts}
  If we run GLOC-TS with $\dot \th_t \sim \cN(\hth_{t-1},\beta^{\ONS}_{t-1}\ibarV_{t-1})$, $\Regret^{\GLOCTS}_T = \hat O\lt(\fr{L(L+R)}{\kap} d^{3/2} \sqrt{T}\log^{3/2}(T)\rt)$ w.h.p.
  If $\eta_t$ is bounded by $\bar R$, then $\hat O\lt(\fr{L(L+\bar R)}{\kap} d^{3/2} \sqrt{T}\log(T)\rt)$.
\end{thm}
Notice that the regret now scales with $d^{3/2}$ as expected from the analysis of linear TS~\cite{agrawal14thompson}, which is higher than scaling with $d$ of GLOC.
This is concerning in the interactive retrieval or product recommendation scenario since the relevance of the shown items is harmed, which makes us wonder if one can improve the regret without loosing the hash-amenability.

%%%%%%%%%%%%%%%%%%%%
\vspace{-6pt}
\paragraph{Quadratic GLOC} 
%%%%%%%%%%%%%%%%%%%%
We now propose a new hash-amenable algorithm called Quadratic GLOC (QGLOC).
Recall that GLOC chooses the arm $\x^{\GLOC}$ by~\eqref{eq:glocopt_ext}.
Define $r = \min_{\x\in\calX} ||\x||_2$ and 
\begin{equation}\begin{aligned} \label{def-m}
    \barm_{t-1} := \min_{\x: ||\x||_2 \in [r,1]} ||\x||_{\ibarV_{t-1}} \;,
\end{aligned}\end{equation}
which is $r$ times the square root of the smallest eigenvalue of $\ibarV_{t-1}$.
It is easy to see that $\barm_{t-1} \le ||\x||_{\ibarV_{t-1}}$ for all $\x\in\cX$ and that $\barm_{t-1}\ge r/\sqrt{t+\lam}$ using the definition of $\barV_{t-1}$.
There is an alternative way to define $\barm_{t-1}$ without relying on $r$, which we present in SM.

Let $c_0 > 0$ be the exploration-exploitation tradeoff parameter (elaborated upon later).
At time $t$, QGLOC chooses the arm
\begin{equation}\begin{aligned}    \label{eq:qglocopt} 
  \x_{t}^{\QGLOC} :=&  \arg \max_{\x\in\calX_t} \la \hatth_{t-1}, \x \ra  + \fr{\beta_{t-1}^{1/4}}{4 c_0 \barm_{t-1}} || \x ||^2_{\barV^{-1}_{t-1}} 
= \arg \max_{\x\in\cX_t} \lt\langle \q_t,  \phi(\x) \rt\rangle \;, 
\end{aligned}\end{equation}
where $\q_t = [ \hth_{t-1}; \vec( \fr{\beta^{1/4}_{t-1} }{4c_0\barm_{t-1}} \ibarV_{t-1} )] \in \dsR^{d+d^2}$ and $\phi(\x) := [\x; \vec(\x\x^\T)]$.
The key property of QGLOC is that the objective function is now quadratic in $\x$, thus the name \emph{Quadratic} GLOC, and can be written as an inner product.
Thus, QGLOC is hash-amenable.
We present the regret bound of QGLOC~\eqref{eq:qglocopt} in Theorem~\ref{thm:x_t9}.
The key step of the proof is that the QGLOC objective function~\eqref{eq:qglocopt} plus $c_0\beta^{3/4} \barm_{t-1}$ is a tight upper bound of the GLOC objective function~\eqref{eq:glocopt_ext}.
\begin{thm}\label{thm:x_t9}
  Run QGLOC with $C^{\ONS}_t$.
  Then, w.p. at least $1-2\dt$,
  \[
    \Regret_T^{\emph{\text{QGLOC}}} = O\lt(  \lt(\fr{1}{c_0}\lt( \fr{L+R}{\kap} \rt)^{1/2} + c_0\lt( \fr{L+R}{\kap} \rt)^{3/2} \rt) Ld^{5/4} \sqrt{T} \log^2(T)\rt)\;.
    \]
   By setting $c_0 = \lt( \fr{L+R}{\kap} \rt)^{-1/2}$, the regret bound is $O(\fr{L(L+R)}{\kap} d^{5/4} \sqrt{T} \log^2(T)) $.
%   Set $c_0 = c_0'\lt(\kap^{-2}(L^2 + R^2)\log(t)\rt)^{-1/4}$, where $c_0'>0$ is an absolute constant.
%   Run QGLOC with $C^{\ONS}_t$.
%   Then, w.p. at least $1-2\dt$,
%     $
%     \emph{\text{Regret}}_T^{\emph{\text{QGLOC}}} = O( \kap^{-1}L(L+R) d^{5/4} \sqrt{T} \log^{7/4}(T) )
%     $.
\end{thm}
%%% BEG: maybe not need
%The key step of the proof is that the maximum of the QGLOC objective function~\eqref{eq:qglocopt} plus $c_0\beta^{3/4} \barm_{t-1}$ is a tight upper bound of the maximum of the GLOC objective function~\eqref{eq:glocopt_ext}.
%%% END
Note that one can have a better dependence on $\log T$ when $\eta_t$ is bounded (available in the proof).
The regret bound of QGLOC is a $d^{1/4}$ factor improvement over that of GLOC-TS; see Table~\ref{tab:bandits}.
Furthermore, in~\eqref{eq:qglocopt} $c_0$ is a free parameter that adjusts the balance between the exploitation (the first term) and exploration (the second term).
Interestingly, the regret guarantee \emph{does not break down} when adjusting $c_0$ in Theorem~\ref{thm:x_t9}.
Such a characteristic is not found in existing algorithms but is attractive to practitioners, which we elaborate in SM.

%%%%%%%%%%%%%%%%%%%%
\vspace{-6pt}
\paragraph{Maximum Inner Product Search (MIPS) Hashing}
%%%%%%%%%%%%%%%%%%%%

While MIPS hashing algorithms such as~\cite{shrivastava14asymmetric,shrivastava15improved,neyshabur15on} can solve~\eqref{eq:qglocopt} in time sublinear in $N$, these necessarily introduce an approximation error.
Ideally, one would like the following guarantee on the error with probability at least $1-\dtH$:
\vspace{-3pt}
\begin{defn}\label{def:cmips}
 Let $\cX \subseteq \dsR^{d'}$ satisfy $|\cX| < \infty$.
 A data point $\til\x \in \cX$ is called $\ccH$-MIPS w.r.t. a given query $\q$ if it satisfies $\la \q, \til\x \ra \ge \ccH \cdot \max_{\x\in\cX}\la \q, \x \ra$ for some $\ccH<1$.
 An algorithm is called $\ccH$-MIPS if, given a query $\q\in {\dsR}^{d'}$, it retrieves $\x \in \cX$ that is $\ccH$-MIPS w.r.t. $\q$.
\end{defn}
\vspace{-5pt}
Unfortunately, existing MIPS algorithms do not directly offer such a guarantee, and one must build a series of hashing schemes with varying hashing parameters like~\citet{indyk12approximate}.
Under the fixed budget setting $T$, we elaborate our construction that is simpler than~\cite{indyk12approximate} in SM.
%%% BEG: some try to explain the hashing
% The way hashing works is that each hash table uses $k$ keys where each key is a discretized value of an inner product between an arm $\x'\in\dsR^{d'}$ and an independent normally-distributed vector. 
% Due to the discretization, $k$ keys becomes a bucket index, and we store item (arm) pointers in the buckets according to their index.
% We build this table $U$ times.
% The key is that given a query $\q$ one can compute its hash keys in a similar way, looks up the matching buckets, and find the inner product maximizer from there.
% We further construct $J$ such hashing schemes for a technical reason.
%%% END
%%% BEG
% Assuming the fixed budget setting with time horizon $T$, we show that the maximimum of~\eqref{eq:qglocopt} is trapped in $[M_{\min},M_{\max}]$ with high probability, where $M_{\min} = O(1)$ and $M_{\max} = \tilde O(d^{1/4}\sqrt{T}))$.
% Then, we build $J := \lcl \log_{1/\sqrt{\ccH}} (M_{\max}/M_{\min}) \rcl  = \hat O(\log (dT) / \log(\ccH^{-1}))$ independent hashings with varying parameters where each has $U$ tables with length-$k$ hash keys.
% Here, $k= O(\log N)$ and $U = O(N^{\rho^*}) $, where $\rho^*$ is an optimized value that is always less than 1. % and depends on $\ccH$.
% Although the dependencies between $(\ccH, \dtH)$ and $(J, U, k)$ are complicated and often omitted here, we remark that as we increase $\ccH$ and reduce $\dtH$ (more accurate) we need to increase $J$, $U$, and $k$ (more space and time).
%%% END

%%%%%%%%%%%%%%%%%%%%%%%%%%%%%%%%%%%%%%%%
\vspace{-6pt}
\paragraph{Time and Space Complexity} % Comparison to Other Hash-Amenable Algorithms}
%%%%%%%%%%%%%%%%%%%%%%%%%%%%%%%%%%%%%%%%

%%% BEG
% The time complexity of our construction with $d'$-dimensional vectors is $O(\log(J) Ukd')$, and the space complexity (except the original data) is $O(JU(N + kd'))$.
% QGLOC uses $d'=d+d^2$.\footnote{
%   Note that this does not mean we need to store $\text{vec}(\x\x^\T)$ since an inner product with it is structured.
% }
% The time complexity of our $\ccH$-MIPS hashing for QGLOC is thus $O(\log(J) N^{\rho^*} \log(N) d^2)$ per query.
% This achieves a sublinear time in $N$.
% %The extra space (except for the original data) required is $JUk$ projection vectors of $(d+d^2)$-dimension plus $O(JUN)$ number of item pointers (indexes) residing in buckets.
% The space complexity is $O(J N^{\rho^*}(N + d^2 \log(N)))$.
% While the time and space complexity grows with the time horizon $T$, the dependence is mild; $\log \log(T)$ and $\log(T)$, respectively.
% 
% GLOC-TS can use a similar hashing scheme.
% GLOC-TS uses $d'=d$, so the time and space complexity is $O(\log(J)N^{\rho^*}\log(N) d)$ and $O(J N^{\rho^*}(N + d \log(N)))$, respectively, which are both a factor-of-$d$ smaller than that of QGLOC. 
% % The extra space complexity of GLOC-TS is $O(J N^{\rho^*}(N + d \log(N)))$, which is also smaller.
% However, GLOC-TS has a worse regret bound than QGLOC.
%%% END
Our construction involves saving Gaussian projection vectors that are used for determining hash keys and saving the buckets containing pointers to the actual arm vectors.
The time complexity for retrieving a $\ccH$-MIPS solution involves determining hash keys and evaluating inner products with the arms in the retrieved buckets.
Let $\rho^*<1$ be an optimized value for the hashing (see~\cite{shrivastava14asymmetric} for detail).
The time complexity for $d'$-dimensional vectors is $O\lt(\log\lt(\fr{\log(dT)}{\log(\ccH^{-1})}\rt) N^{\rho^*} \log(N) d'\rt)$, and the space complexity (except the original data) is $O\lt(\fr{\log(dT)}{\log(\ccH^{-1})} N^{\rho^*}(N + \log(N)d')\rt)$.
While the time and space complexity grows with the time horizon $T$, the dependence is mild; $\log \log(T)$ and $\log(T)$, respectively.
QGLOC uses $d'=d+d^2$,\footnote{
  Note that this does not mean we need to store $\text{vec}(\x\x^\T)$ since an inner product with it is structured.
}
  and GLOC-TS uses $d'=d'$.
While both achieve a time complexity sublinear in $N$, the time complexity of GLOC-TS scales with $d$ that is better than scaling with $d^2$ of QGLOC.
However, GLOC-TS has a $d^{1/4}$-factor worse regret bound than QGLOC.

%%%%%%%%%%%%%%%%%%%%
\vspace{-6pt}
\paragraph{Discussion}
%%%%%%%%%%%%%%%%%%%%

While it is reasonable to incur small errors in solving the arm selection criteria like~\eqref{eq:qglocopt} and sacrifice some regret in practice, the regret bounds of QGLOC and GLOC-TS do not hold anymore. % presence of the errors break the regret bound.
Though not the focus of our paper, we prove a regret bound under the presence of the hashing error in the fixed budget setting for QGLOC; see SM.
Although the result therein has an inefficient space complexity that is linear in $T$, it provides the first low regret bound with time sublinear in $N$, to our knowledge.

%%%%%%%%%%%%%%%%%%%%%%%%%%%%%%%%%%%%%%%%%%%%%%%%%%%%%%%%%%%%%%%%%%%%%%%%%%%%%%%%
\vspace{-4pt}
\section{Approximate Inner Product Computations with L1 Sampling}
\label{sec:iprod}
\vspace{-4pt}
%%%%%%%%%%%%%%%%%%%%%%%%%%%%%%%%%%%%%%%%%%%%%%%%%%%%%%%%%%%%%%%%%%%%%%%%%%%%%%%%

While hashing allows a time complexity sublinear in $N$, it performs an additional computation for determining the hash keys.
Consider a hashing with $U$ tables and length-$k$ hash keys. 
Given a query $\q$ and projection vectors $\a^{(1)}, \ldots, \a^{(Uk)}$, the hashing computes $\q^\T \a^{(i)}$, $\forall i \in [Uk]$ to determine the hash key of $\q$. 
To reduce such an overhead, approximate inner product methods like~\cite{jain10hashing,kannan09spectral} are attractive since hash keys are determined by discretizing the inner products; small inner product errors often do not alter the hash keys. 

In this section, we propose an improved approximate inner product method called \emph{L1 sampling} which we claim is more accurate than the sampling proposed by~\citet{jain10hashing}, which we call \emph{L2 sampling}.
Consider an inner product $\q^\T \a$.
The main idea is to construct an unbiased estimate of $\q^\T \a$.
That is, let $\p \in \dsR^d$ be a probability vector.
Let
\begin{equation}\begin{aligned} \label{eq:def_iprod} 
  i_k \stackrel{\text{i.i.d.}}{\sim} \text{Multinomial}(\p) \quad \text{ and } \quad
  G_k := q_{i_k} a_{i_k} / p_{i_k}, \; k \in [m]  \;.
\end{aligned}\end{equation}
It is easy to see that $\expt G_k = \q^\T \a$.
By taking $\tfrac{1}{m}\sum_{k=1}^m G_k$ as an estimate of $\q^\T\a$, the time complexity is now $O(mUk)$ rather than $O(d'Uk)$.
The key is to choose the right $\p$.
L2 sampling uses $\p^{(\tL2)} := [q_i^2 / ||\q||_2^2]_i$.
Departing from L2, we propose $\p^{(\tL1)}$ that we call L1 sampling and define as follows:
\begin{align}
  \p^{(\tL1)} := [|q_1|;\cdots;|q_{d'}|] / || \q ||_1 \;.
\end{align}
We compare L1 with L2 in two different point of view.
Due to space constraints, we summarize the key ideas and defer the details to SM.

The first is on their concentration of measure.
Lemma~\ref{lem:p_l1} below shows an error bound of L1 whose failure probability decays exponentially in $m$.
This is in contrast to decaying polynomially of L2~\cite{jain10hashing}, which is inferior.\footnote{
  In fact, one can show a bound for L2 that fails with exponentially-decaying probability. However, the bound introduces a constant that can be arbitrarily large, which makes the tails thick. We provide details on this in SM.
}       
\begin{lem}\label{lem:p_l1}
  Define $G_k$ as in~\eqref{eq:def_iprod} with $\p=\p^{(\emph\tL1)}$.
  Then, given a target error $\eps>0$,
  \begin{equation}\begin{aligned} \label{eq:lem-p_l1}
      \textstyle   \P\lt( \lt|\fr{1}{m}\sum_{k=1}^m G_k - \q^\T\a \rt| \ge \eps \rt) \le  2\exp\lt(-\fr{m\eps^2}{2||\q||_1^2||\a||_{\max}^2}\rt)
  \end{aligned}\end{equation}
\end{lem}
To illustrate such a difference, we fix $\q$ and $\a$ in 1000 dimension and apply L2 and L1 sampling 20K times each with $m=5$ where we scale down the L2 distribution so its variance matches that of L1.
Figure~\ref{fig:iprod}(a) shows that L2 has thicker tails than L1.
Note this is not a pathological case but a typical case for Gaussian $\q$ and $\a$.
This confirms our claim that L1 is safer than L2. 

\begin{wrapfigure}{R}{0.51\textwidth}
  \vspace{-21pt}
%  \vspace{-18pt}
\begin{minipage}{0.51\textwidth}
\begin{figure}[H]
  \begin{center}
{  \centering \footnotesize
\begin{tabular}{cc}
  \hspace{-12pt}
  \includegraphics[width=.50\textwidth]{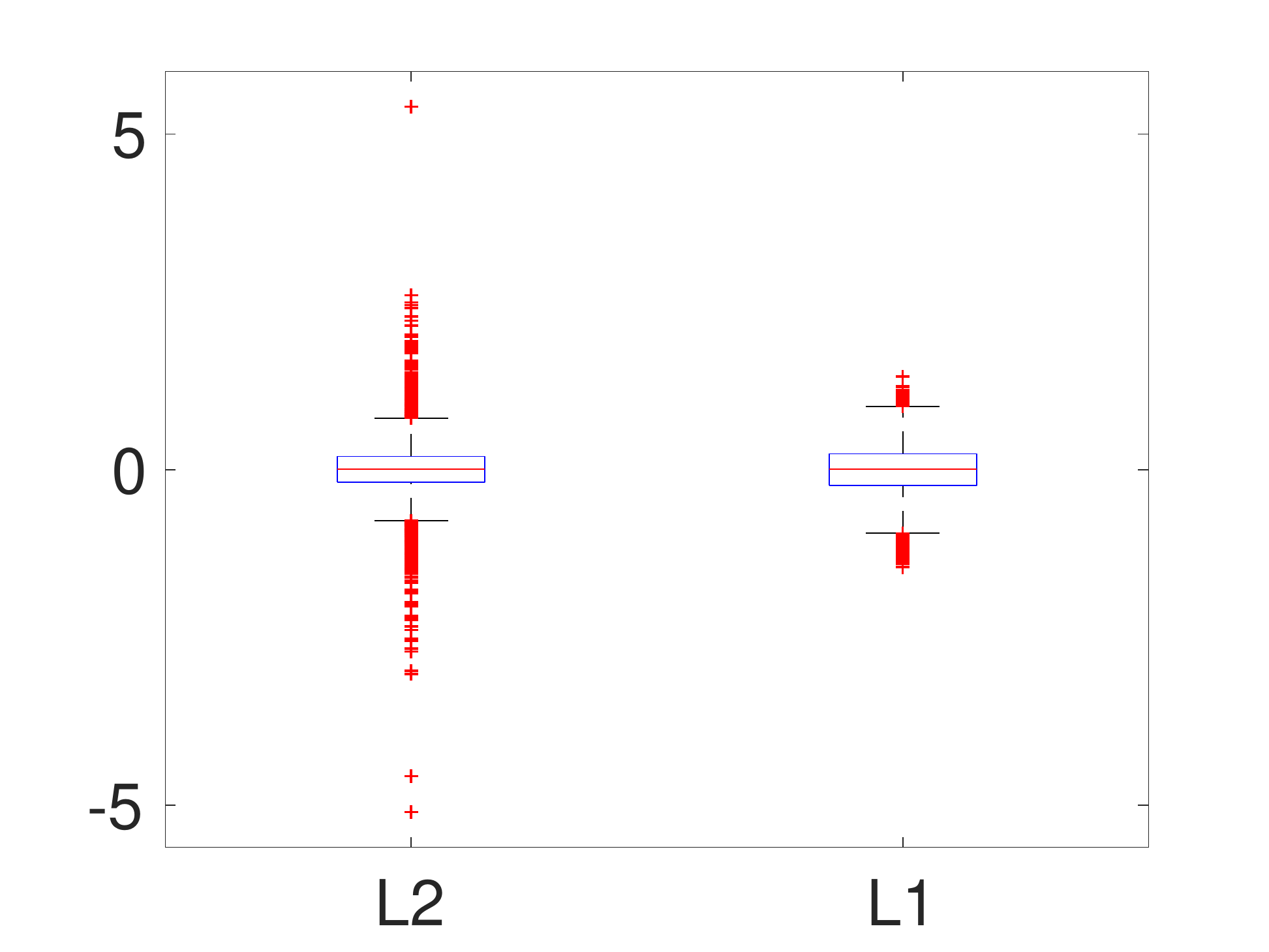} & \hspace{-10pt}
  \includegraphics[width=.50\textwidth]{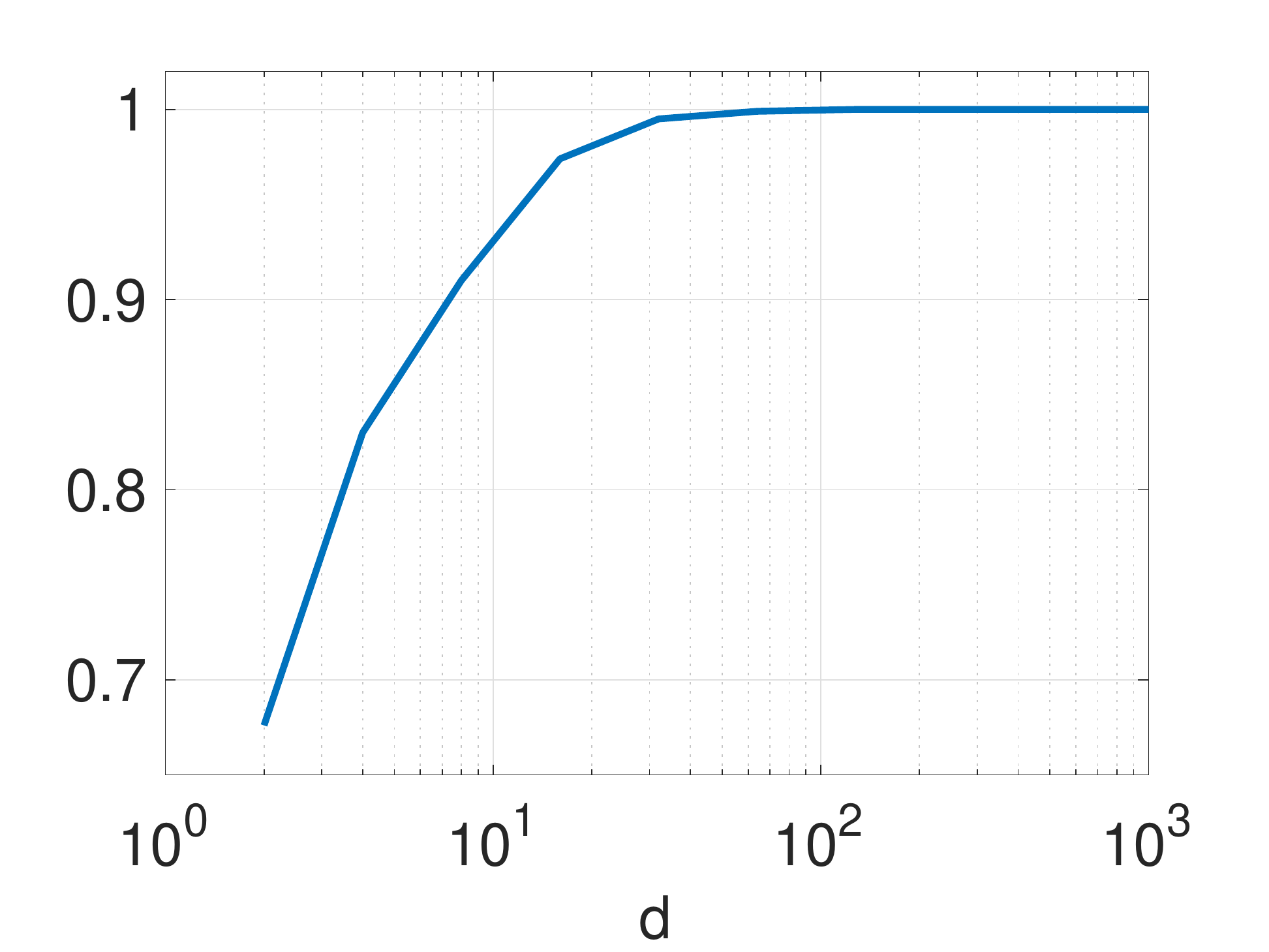} \vspace{-3pt}\\
  (a) & (b) \\
\end{tabular}
\vspace{-10pt}
}
\end{center}
\end{figure}
\end{minipage}
\vspace{-8pt}
\caption{%\hspace{-12pt}
(a) A box plot of estimators.
L1 and L2 have the same variance, but L2 has thicker tails.
(b) The frequency of L1 inducing smaller variance than L2 in 1000 trials. After 100 dimensions, L1 mostly has smaller variance than L2. 
}
\label{fig:iprod}
\vspace{-5pt}
\end{wrapfigure}

Another point of comparison is the variance of L2 and L1.
We show that the variance of L1 may or may not be larger than L2 in SM; there is no absolute winner.
However, if $\q$ and $\a$ follow a Gaussian distribution, then L1 induces smaller variances than L2 for large enough $d$; see Lemma~\ref{lem:l1-var} in SM.
Figure~\ref{fig:iprod}(b) confirms such a result.
The actual gap between the variance of L2 and L1 is also nontrivial under the Gaussian assumption.
For instance, with $d=200$, the average variance of $G_k$ induced by L2 is 0.99 whereas that induced by L1 is 0.63 on average.
Although a stochastic assumption on the vectors being inner-producted is often unrealistic, in our work we deal with projection vectors $\a$ that are truly normally distributed.

\vspace{-4pt}
\section{Experiments}
\label{sec:expr}
\vspace{-4pt}

\begin{figure}
% \vspace{-20pt}
  \begin{center}
{  \centering \footnotesize
\begin{tabular}{ccc}
  \hspace{-5pt}\imagetop{\includegraphics[width=.29\textwidth,valign=top]{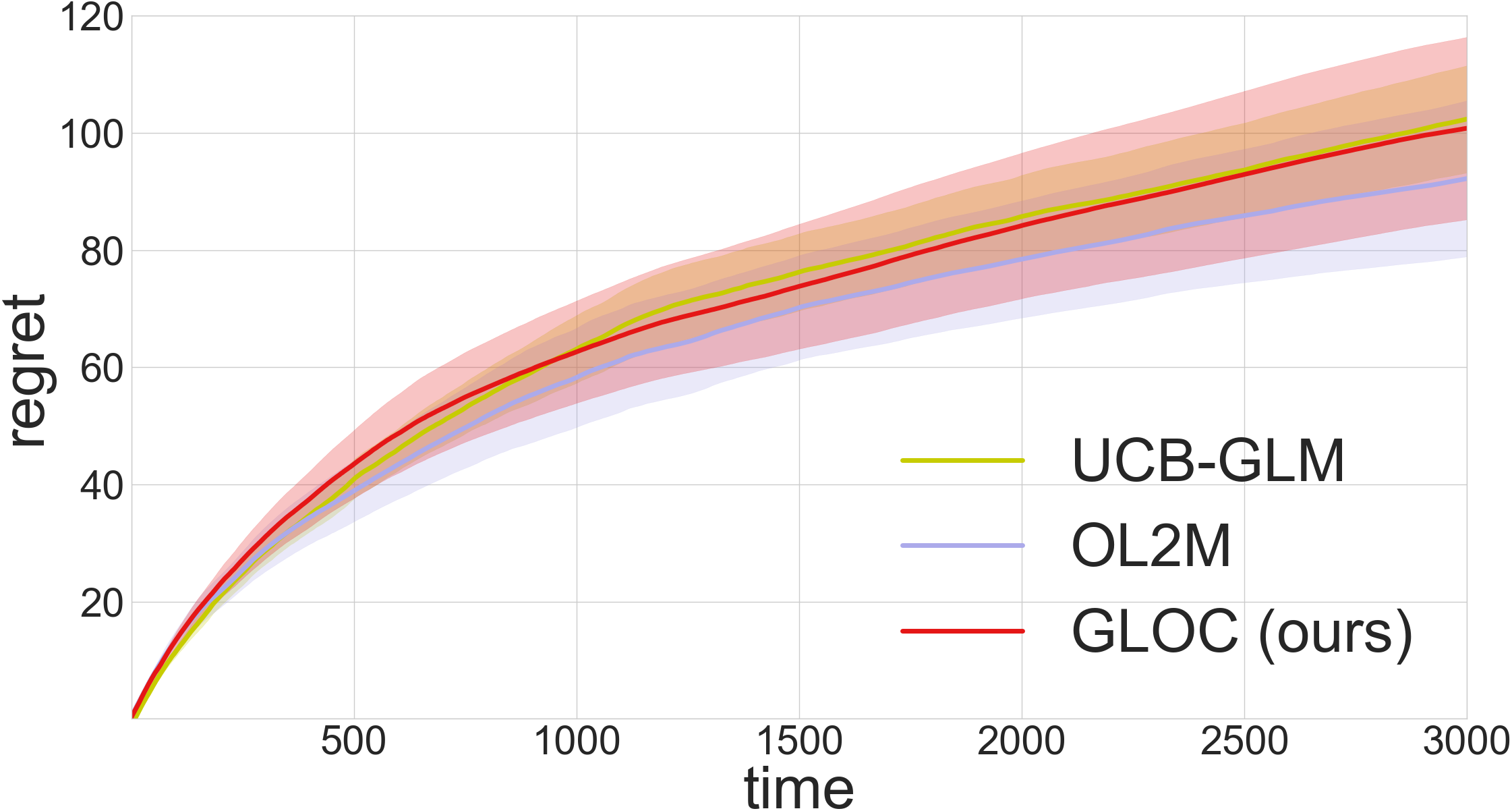}} &
  \imagetop{\includegraphics[width=.29\textwidth,valign=top]{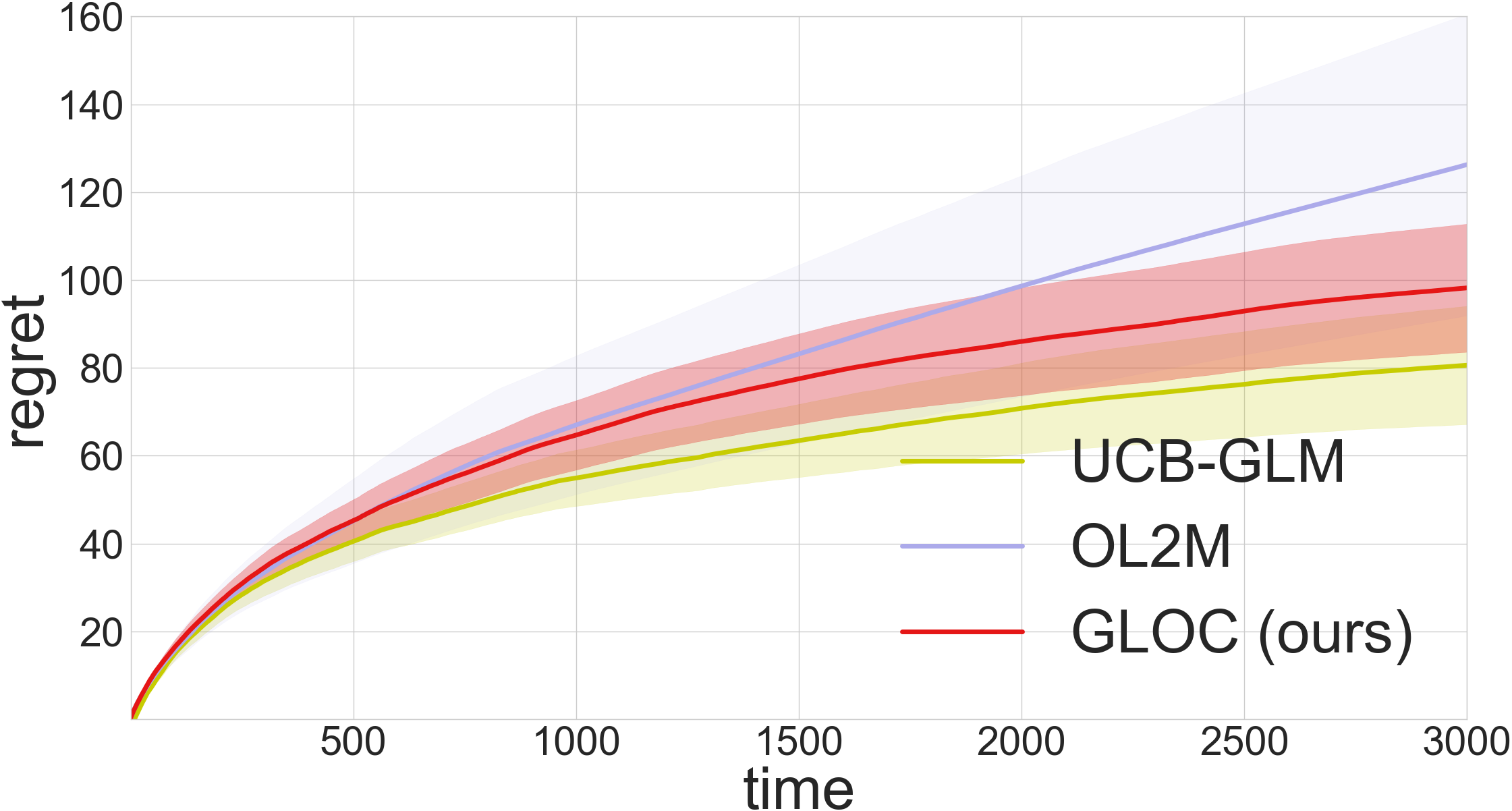}} \vspace{5pt} &
  \imagetop{\begin{tabular}{lc} \hline
    Algorithm     & Cum. Regret \\ \hline
%    GLOC          & 270.2 ($\pm$20.5)  \\
%    GLOC-Lite     & 305.1 ($\pm$15.9)  \\
    QGLOC         & 266.6 ($\pm$19.7)  \\
%    QGLOC-Lite    & 302.9 ($\pm$24.2)  \\
    QGLOC-Hash    & 285.0 ($\pm$30.3)  \\
    GLOC-TS       & 277.0 ($\pm$36.1)  \\
%    GLOC-TS-Lite  & 321.9 ($\pm$20.2)  \\
    GLOC-TS-Hash  & 289.1 ($\pm$28.1)  \\ \hline
  \end{tabular}} 
\\ (a) & (b) & (c)
\end{tabular}
}
\end{center}
\vspace{-18pt}
\caption{
  Cumulative regrets with confidence intervals under the (a) logit and (b) probit model.
  (c) Cumulative regrets with confidence intervals of hash-amenable algorithms.
}
\label{fig:true_regret} 
\vspace{-13pt}
\end{figure}

We now show our experiment results comparing GLB algorithms and hash-amenable algorithms.

\vspace{-6pt}
\paragraph{GLB Algorithms}
We compare GLOC with two different algorithms: UCB-GLM~\cite{li17provable} and Online Learning for Logit Model (OL2M)~\cite{zhang16online}.\footnote{We have chosen UCB-GLM over GLM-UCB of~\citet{filippi10parametric} as UCB-GLM has a lower regret bound.} % while the two are very similar. } 
For each trial, we draw $\th^*\in\dsR^d$ and $N$ arms ($\cX$) uniformly at random from the unit sphere.
We set $d=10$ and $\cX_t = \cX$, $\forall t\ge1$.
Note it is a common practice to scale the confidence set radius for bandits~\cite{chapelle11anempirical,li12anunbiased}.
Following~\citet{zhang16online}, for OL2M we set the squared radius $\gamma_t = c \log(\det(\Z_t)/\det(\Z_1))$, where $c$ is a tuning parameter.
For UCB-GLM, we set the radius as $\alpha = \sqrt{c d\log t}$.
For GLOC, we replace $\beta^\ONS_t$ with $c \sum_{s=1}^t g_s^2 ||\x_s||^2_{\A_s^{-1}}$.
While parameter tuning in practice is nontrivial, for the sake of comparison we tune $c \in \{10^1, 10^{0.5}, \ldots, 10^{-3}\}$ and report the best one.
We perform 40 trials up to time $T=3000$ for each method and compute confidence bounds on the regret.

We consider two GLM rewards: $(i)$ the logit model (the Bernoulli GLM) and $(ii)$ the probit model (non-canonical GLM) for 0/1 rewards that sets $\mu$ as the probit function.
Since OL2M is for the logit model only, we expect to see the consequences of model mismatch in the probit setting. 
For GLOC and UCB-GLM, we specify the correct reward model.
We plot the cumulative regret under the logit model in Figure~\ref{fig:true_regret}(a).
All three methods perform similarly, and we do not find any statistically significant difference based on paired t test.
The result for the probit model in Figure~\ref{fig:true_regret}(b) shows that OL2M indeed has higher regret than both GLOC and UCB-GLM due to the model mismatch in the probit setting.
Specifically, we verify that at $t=3000$ the difference between the regret of UCB-GLM and OL2M is statistically significant.
Furthermore, OL2M exhibits a significantly higher variance in the regret, which is unattractive in practice.
This shows the importance of being generalizable to \emph{any} GLM reward. 
Note we observe a big increase in running time for UCB-GLM compared to OL2M and GLOC.

\vspace{-6pt}
\paragraph{Hash-Amenable GLBs}\quad
To compare hash-amenable GLBs, we use the logit model as above but now with $N$=100,000 and $T$=5000.
We run QGLOC, QGLOC with hashing (QGLOC-Hash), GLOC-TS, and GLOC-TS with hashing (GLOC-TS-Hash), where we use the hashing to compute the objective function (e.g.,~\eqref{eq:qglocopt}) on just 1\% of the data points and save a significant amount of computation.
Details on our hashing implementation is found in SM.
Figure~\ref{fig:true_regret}(c) summarizes the result. 
%%% BEG
%where we also run GLOC (not hash-amenable) for the sake of comparison.
%We observe that both QGLOC and GLOC-TS are comparable to GLOC while QGLOC % has the smallest regret despite the fact that the computation was performed only on . 
%%% END
We observe that QGLOC-Hash and GLOC-TS-Hash increase regret from QGLOC and GLOC-TS, respectively, but only moderately, which shows the efficacy of hashing.

%%%%%%%%%%%%%%%%%%%%%%%%%%%%%%%%%%%%%%%%%%%%%%%%%%%%%%%%%%%%%%%%%%%%%%%%%%%%%%%%
\vspace{-4pt}
\section{Future Work}
\label{sec:conclusion}
\vspace{-4pt}
%%%%%%%%%%%%%%%%%%%%%%%%%%%%%%%%%%%%%%%%%%%%%%%%%%%%%%%%%%%%%%%%%%%%%%%%%%%%%%%%

In this paper, we have proposed scalable algorithms for the GLB problem: $(i)$ for large time horizon $T$ and $(ii)$ for large number $N$ of arms.
%%% BEG
% For large time horizon $T$, unlike existing GLBs, our new algorithmic framework GLOC performs online computations to enjoy constant per-time-step space and time complexity.
% The regret bound of GLOC combined with a GLM-specialized version of ONS achieves a comparable regret bound to existing ones.
% For large number $N$ of arms, we proposed two new hash-amenable algorithms GLOC-TS and QGLOC, where QGLOC (regret bound $\tilde O(d^{5/4}T)$) achieves a better regret bound than GLOC-TS ($\tilde O(d^{3/2}T)$).
% We proposed L1 sampling for fast inner product computation that is more accurate than the existing state-of-the-art. 
%%% END
There exists a number of interesting future work.
First, we would like to extend the GLM rewards to the single index models~\cite{kalai09theisotron} so one does not need to know the function $\mu$ ahead of time under mild assumptions.
Second, closing the regret bound gap between QGLOC and GLOC without loosing hash-amenability would be interesting: i.e., develop a hash-amenable GLB algorithm with $O(d\sqrt{T})$ regret.
In this direction, a first attempt could be to design a hashing scheme that can directly solve~\eqref{eq:glocopt_ext} approximately.

%%% BEG
%- 1. can we even avoid the projection step? this is iterative process. projection free methods can be useful
%%% END

%%%%%%%%%%%%%%%%%%%%%%%%%%%%%%%%%%%%%%%%%%%%%%%%%%%%%%%%%%%%%%%%%%%%%%%%%%%%%%%%
% END OF THE MAIN TEXT
%%%%%%%%%%%%%%%%%%%%%%%%%%%%%%%%%%%%%%%%%%%%%%%%%%%%%%%%%%%%%%%%%%%%%%%%%%%%%%%%

%\kw{{(placeholder citation~\citet{ay11improved})}}

%\newpage
\section*{Acknowledgments}
This work was partially supported by the NSF grant IIS-1447449 and the MURI grant 2015-05174-04. The authors thank Yasin Abbasi-Yadkori and Anshumali Shrivastava for providing constructive feedback and Xin Hunt for her contribution at the initial stage.
%\section*{References}
%\bibliographystyle{IEEEannot} %\bibliographystyle{IEEE}
%\clearpage
{\small
\bibliographystyle{icml2017_kwang}
\bibliography{library-shared}
}

%%%%%%%%%%%%%%%%%%%%%%%%%%%%%%%%%%%%%%%%%%%%%%%%%%%%%%%%%%%%%%%%%%%%%%%%%%%%%%%%
%- supplementary material
%%%%%%%%%%%%%%%%%%%%%%%%%%%%%%%%%%%%%%%%%%%%%%%%%%%%%%%%%%%%%%%%%%%%%%%%%%%%%%%%
\clearpage

%\onecolumn
\begin{center}
{\Large\bf Supplementary Material (Appendix)}
\end{center}
\renewcommand\thesection{\Alph{section}}
\setcounter{section}{0}

%%%%%%%%%%%%%%%%%%%%%%%%%%%%%%%%%%%%%%%%%%%%%%%%%%%%%%%%%%%%%%%%%%%%%%%%%%%%%%%%
\vspace{-5pt}
\section{Proof of Theorem~\ref{thm:o2cs}}
\vspace{-5pt}
%%%%%%%%%%%%%%%%%%%%%%%%%%%%%%%%%%%%%%%%%%%%%%%%%%%%%%%%%%%%%%%%%%%%%%%%%%%%%%%%
%
We first describe the sketch of the proof.
We perform a generalized version of the online-to-confidence-set conversion of~\citet[Theorem 1]{ay12online} that was strictly for linear rewards. 
Unlike that work, which deals with the squared loss, we now work with the negative log likelihood loss.
The key is to use the strong-convexity of the loss function to turn the OL regret bound~\eqref{eq:ol_regret} into a quadratic equation in $\th^*$. 
Then, it remains to bound random quantities with a Martingale concentration bound.
\vspace{-5pt}
\begin{proof}
%  We override the notation $\ell_s(\th)$ and denote by $\ell(z,y_s)$
  Let $\ell'(z,y)$ and $\ell''(z,y)$ be the first and second derivative of the loss $\ell$ w.r.t. $z$.
  We lower bound the LHS of~\eqref{eq:ol_regret} using Taylor's theorem with some $\xi_s$ between $\x_s^\T\th^*$ and $\x_s^\T\th_s$:
  \vspace{-4pt}
  \begin{equation*}\begin{aligned}
      B_t 
      \ge \sum_{s=1}^t \ell_s(\th_s) - \ell_s(\th^*) 
      &= \sum_{s=1}^t \ell'(\x_s^\T\th^*, y_s) \x_s^\T(\th_s - \th^*) + \fr{\ell''(\xi_s,y_s)}{2} (\x_s^\T(\th_s - \th^*))^2
  \\&\stackrel{\text{(\text{Assumption~\ref{ass:mu}})}}{\ge} 
         \sum_{s=1}^t \ell'(\x_s^\T\th^*,y_s) \x_s^\T(\th_s - \th^*) + \fr{\kappa}{2} (\x_s^\T(\th_s - \th^*))^2  \;.
  \end{aligned} \end{equation*}
  Since $\ell'(\x_s^\T\th^*,y_s) = -y_s + \mu(\x_s^\T\th^*) = -\eta_s$,
  \vspace{-4pt}
  \begin{equation*}\begin{aligned}
    \sum_{s=1}^t (\x_s^\T(\th_s - \th^*))^2  \le \fr{2}{\kappa} B_t + \fr{2}{\kappa}\sum_{s=1}^t \eta_s \lt(\x_s^\T(\th_s - \th^*)\rt) \;.
  \end{aligned}
  \end{equation*}
  Note that the second term in the RHS involves $\eta_t$ that is unknown and random, which we bound using~\citet[Corollary 8]{ay12online}. 
%  Note that the only unobserved quantity except for $\th^*$ is the factor $\eta_t$ in the second term on the RHS, which we bound using~\citet[Corollary 8]{ay12online}.
  That is, w.p. at least $1-\dt$, for all $t\ge1$,
  \vspace{-4pt}
  \begin{equation*}
    \sum_{s=1}^t \eta_s  \lt(\x_s^\T(\th_s - \th^*)\rt) 
    \le R\sqrt{\lt(2 + 2\sum_{s=1}^t (\x_s^\T(\th_s-\th^*))^2\rt) \cdot\log\lt( \fr{1}{\dt} \sqrt{1+\sum_{s=1}^t (\x_s^\T(\th_s-\th^*))^2 } \rt)} .
  \end{equation*}
  Then,
  \vspace{-4pt}
  \begin{equation*}\begin{aligned}
      \sum_{s=1}^t \lt(\x_s^\T(\th_s - \th^*) \rt)^2 \le \fr{2}{\kappa}B_t + \fr{2 R}{\kappa} \sqrt{\lt(2 + 2\sum_{s=1}^t (\x_s^\T(\th_s-\th^*))^2\rt) \cdot\log\lt( \fr{1}{\dt} \sqrt{1+\sum_{s=1}^t (\x_s^\T(\th_s-\th^*))^2 } \rt)} . 
  \end{aligned}\end{equation*}
  Define $q := \sqrt{1 + \sum_{s=1}^t (\x_s^\T(\th_s-\th^*))^2}$. Then, the inequality above can be written as $ q^2 \le 1+ \fr{2}{\kappa}B_t + \fr{2\sqrt{2} R}{\kappa} q \sqrt{\log(q/\dt)} $.
  The following Lemma is useful. See Section~\ref{sec:proof_lem_arith} for a proof.
  \vspace{-4pt}
  \begin{lem}\label{lem:arith}
    Let $\dt\in(0,1), a \ge 0, f \ge 0, q \ge 1$.
    Then, 
    \[
      q^2 \le a + f q \sqrt{\log\lt(\fr{q}{\dt}\rt)} 
      \implies q^2 \le 2a + f^2\log\lt(\fr{\sqrt{4a + f^4/(4\dt^2)}}{\dt}\rt)
    \]
  \end{lem}
  \vspace{-4pt}
  Applying Lemma~\ref{lem:arith} with $a := 1+ \fr{2}{\kappa}B_t$ and $f := 2\sqrt{2}R/\kappa$, we have
  \vspace{-4pt}
  \begin{equation*}\begin{aligned}
      \sum_{s=1}^t(\x_s^\T(\th_s-\th^*))^2 \le 1 + \fr{4}{\kappa}B_t + \fr{8R^2}{\kappa^2} \log\lt(\fr{1}{\dt}\sqrt{ 4+ \fr{8}{\kappa}B_t + \fr{64R^4}{\kappa^4\cdot 4\dt^2} } \rt)  = \beta'_t
  \end{aligned}\end{equation*} %\fr{64R^4}{\kappa^4\cdot 4\dt^2}
  Then, one can rewrite the above as
  \begin{equation*}
    || \z_{t} - \X_{t}\th^* ||^2_2 \le \beta'_t \;.
  \end{equation*}
  %Note that one can rewrite $\y$
  Let $\lam >0$.
  We add $\lam ||\th^*||$ to the both sides. % of~\eqref{eq:thm_o2cs}:
  \begin{equation*}\begin{aligned}
      \lam ||\th^*||^2_2 + || \z_{t} - \X_{t}\th^* ||^2_2 \le \lam ||\th^*||^2_2 + \beta'_t \le \lam S^2 + \beta'_t \;.
  \end{aligned}\end{equation*}
  Hereafter, we omit $t$ from $\X_t $ and $\z_t $ for brevity.
  Since the LHS is quadratic in $\th^*$, we can rewrite it as an ellipsoid centered at $\hth_t := \arg\min_{\th} \lam ||\th||^2_2 + ||\z  - \X \th||^2_2 = \ibarV_t\X^\T \z$ as follows:
  \begin{equation*}\begin{aligned}
     ||\th^* - \hth_t||^2_{\barV_t} + \underbrace{\lam ||\hth_t||^2_2 + ||\z  - \X \hth_t ||^2_2}_{= ||\z ||^2_2 - \hth_t^\T\X^\T\z } \le \lam S^2 + \beta'_t \;,
  \end{aligned}\end{equation*}
  which concludes the proof.
\end{proof}

%%%%%%%%%%%%%%%%%%%%%%%%%%%%%%%%%%%%%%%%%%%%%%%%%%%%%%%%%%%%%%%%%%%%%%%%%%%%%%%%
\vspace{-8pt}
\subsection{Proof of Lemma~\ref{lem:arith}}
\label{sec:proof_lem_arith}
%%%%%%%%%%%%%%%%%%%%%%%%%%%%%%%%%%%%%%%%%%%%%%%%%%%%%%%%%%%%%%%%%%%%%%%%%%%%%%%%

\begin{proof}
Let $c := f\sqrt{\log(q/\dt)}$. Then, $q^2 \le a + cq \implies q^2 -cq -a \le 0$. Solving it for $q$, we get $q \le \fr{c + \sqrt{c^2 + 4a}}{2}$.
Then, using $(u+v)^2 \le 2(u^2 + v^2)$,
\vspace{-4pt}
\begin{equation}\begin{aligned} \label{eq:lem_arith}
  q^2 &\le \lt(\fr{c + \sqrt{c^2 + 4a}}{2}\rt)^2 \le  \fr{2(c^2 + c^2 + 4a)}{4}  = c^2 + 2a  
\\\iff q^2 &\le 2a + f^2 \log(q/\dt)
\end{aligned}\end{equation}
One might suspect that $c$ has $q$ in it, which might cause a problem.
To be assured, one can prove the contrapositive: $q > \fr{c + \sqrt{c^2 + 4a}}{2} \implies q^2-cq-a>0$.
To see this, $q^2-cq-a = (q - \fr{c+\sqrt{c^2+4a}}{2})(q - \fr{c-\sqrt{c^2+4a}}{2}) $ and since $ q > \fr{c+\sqrt{c^2+4a}}{2}$ it suffices to show that $q - \fr{c-\sqrt{c^2+4a}}{2}  > 0$.
Then, since $q - \fr{c-\sqrt{c^2+4a}}{2} \ge q - \fr{c+\sqrt{c^2+4a}}{2} > 0$.

Using $\log u \le \fr{1}{2} u$, 
\vspace{-4pt}
\[
  q^2 \le 2a + f^2 \log(q/\dt) \le 2a + \fr{f^2 }{2\dt}q  \iff q^2 - \fr{f^2}{2\dt}q - 2a \le 0\;.
\]
Solving the quadratic inequality for $q$, we have $q \le \fr{f^2/(2\dt) + \sqrt{(f^4/(4\dt^2)) + 8a}}{2}$.
This implies that 
\vspace{-4pt}
\[
  q^2 \le \fr{2 (f^4/(4\dt^2) + f^4/(4\dt^2) + 8a)}{4} = \fr{f^4}{4\dt^2} + 4a \;.
\]
Now, applying this inequality on $q$ in the RHS of~\eqref{eq:lem_arith},
\[
  q^2 \le 2a + f^2\log\lt(\fr{\sqrt{4a + f^4/(4\dt^2)}}{\dt}\rt)
\]
\end{proof}         
\vspace{-10pt}

%%% BEG
% %%%%%%%%%%%%%%%%%%%%%%%%%%%%%%%%%%%%%%%%%%%%%%%%%%%%%%%%%%%%%%%%%%%%%%%%%%%%%%%%
% \vspace{-4pt}
% \section{Proof of Corollary~\ref{cor:o2cs}}
% \vspace{-4pt}
% %%%%%%%%%%%%%%%%%%%%%%%%%%%%%%%%%%%%%%%%%%%%%%%%%%%%%%%%%%%%%%%%%%%%%%%%%%%%%%%%
% \begin{proof}
%   Then, one can rewrite~\eqref{eq:thm_o2cs} as
%   \begin{equation*}
%     || \z_{t} - \X_{t}\th^* ||^2_2 \le \beta'_t \;.
%   \end{equation*}
%   %Note that one can rewrite $\y$
%   Let $\lam >0$.
%   We add $\lam ||\th^*||$ to the both sides of~\eqref{eq:thm_o2cs}:
%   \begin{equation*}\begin{aligned}
%       \lam ||\th^*||^2_2 + || \z_{t} - \X_{t}\th^* ||^2_2 \le \lam ||\th^*||^2_2 + \beta'_t \le \lam S^2 + \beta'_t \;.
%   \end{aligned}\end{equation*}
%   Hereafter, we omit $t$ from $\X_t $ and $\z_t $ for brevity.
%   Since the LHS is quadratic in $\th^*$, we can rewrite it as an ellipsoid centered at $\hth_t := \arg\min_{\th} \lam ||\th||^2_2 + ||\z  - \X \th||^2_2 = \ibarV_t\X^\T \z$ as follows:
%   \begin{equation*}\begin{aligned}
%      ||\th^* - \hth_t||^2_{\barV_t} + \underbrace{\lam ||\hth_t||^2_2 + ||\z  - \X \hth_t ||^2_2}_{= ||\z ||^2_2 - \hth_t^\T\X^\T\z } \le \lam S^2 + \beta'_t \;,
%   \end{aligned}\end{equation*}
% which concludes the proof.
% \end{proof}
%%% END

%%%%%%%%%%%%%%%%%%%%%%%%%%%%%%%%%%%%%%%%%%%%%%%%%%%%%%%%%%%%%%%%%%%%%%%%%%%%%%%%
\vspace{-5pt}
\section{Proof of Theorem~\ref{thm:regret_o2cs}}
\vspace{-5pt}
%%%%%%%%%%%%%%%%%%%%%%%%%%%%%%%%%%%%%%%%%%%%%%%%%%%%%%%%%%%%%%%%%%%%%%%%%%%%%%%%
\begin{proof}
  Our proof closely follow a standard technique (cf.~\citet{ay11improved}).
  Define $\x_{t,*} = \arg\max_{\x\in\cX_t} \la\x,\th^*\ra$.
  Let $r_t := \mu(\x_{t,*}^\T \th^*) - \mu(\x_t^\T\th^*)$ be the instantaneous regret.
  Using $\mu(\x_{t,*}^\T \th^*) - \mu(\x_t^\T\th^*) \le L(\x_{t,*}^\T \th^* - \x_t^\T\th^*)$,
  \begin{equation*}\begin{aligned}
      \fr{r_t}{L} &\le \x_{t,*}^\T \th^* - \x_t^\T\th^*
  \\&\le \x_t^\T\tilth_t - \x_t^\T\th^*
  \\&=   \x_t^\T(\tilth_t - \hth_{t-1}) + \x_t^\T(\hth_{t-1} - \th^*)
  \\&\le ||\x_t||_{\ibarV_{t-1}} || \tilth_t - \hth_{t-1} ||_{\barV_{t-1}} + ||\x_t||_{\ibarV_{t-1}} || \th^* - \hth_{t-1} ||_{\barV_{t-1}}
  \\&\le 2\sqrt{\bar\beta_t}||\x_t||_{\ibarV_{t-1}} \;.
  \end{aligned}\end{equation*}

  Note that
  \begin{equation}\begin{aligned}
  %E_1(\dt) \implies ||\X_{t}^\T \bfeta_{t} ||_{\ibarV_{t}} \le R\sqrt{d \log\lt( \fr{1+t/(d\lam)}{\dt}\rt)} \;, \label{eq:bound_X_eta} \\
  \sum_{t=1}^T \log\lt(1+ ||\x_t||^2_{\ibarV_{t-1}}\rt) 
     =   \log\lt(\fr{\det(\barV_T)}{\det(\lam \I)}\rt) 
     \le d \log\lt( 1+T/(d\lam)\rt) \;, \label{eq:bound_log_1_x}
  \end{aligned}\end{equation}
  which is due to~\citet[Lemma 11]{ay11improved}. 
  %For~\eqref{bound_log_1_x}, we emphasize that $\barV_{t-1}$ must be constructed with $\x_1,\ldots,\x_{t-1}$.

The following lemmas become useful.
\begin{lem}\label{lem:x_logx}
  For any $q, x \ge 0$,
  $$\min \{q, x\} \le \max\{2, q\} \log (1+x)$$
\end{lem}
\begin{proof}
  It is not hard to see that
\begin{equation}\begin{aligned} \label{x_logx_bound} 
    x \in [0,a] \implies x \le \fr{a}{\log(1+a)} \log(1+x)
\end{aligned}\end{equation}
  We consider the following two cases.

  \textbf{Case 1. } $q \le 2$ \\
  If $x \le 2$, by~\eqref{x_logx_bound}, $\min\{2, x\}  = x \le \fr{2}{\log(3)} \log(1+x) \le 2 \log(1+x)$.
  If $x > 2$, $\min\{2, x\} = 2 \le 2 \log(1+2) \le 2 \log(1+x) $.
  Thus, for any $x$, $\min\{q, x\} \le \min\{2, x\} \le 2 \log(1+x)$.

  \textbf{Case 2. } $q > 2$ \\
  If $x \le q$, by~\eqref{x_logx_bound}, $\min\{q,x\} = x \le \fr{q}{\log(1+q)} \log(1+x) < q \log(1+x)$.
  If $x > q$, $\min\{q,x\} = q \le q \log(1 + 2) \le q \log(1+x)$.

  Combining both cases to complete the proof.
\end{proof}

\begin{lem}\label{lem:regret_useful}
  If $A$ is a value independent of $t$,
  \[
    \sum_{t=1}^T \min \{A, ||\x_t||^2_{\ibarV_{t-1}}\} \le \max\{2,A\} d\log(1 + T/(d\lam)) \;.
  \]
\end{lem}
\begin{proof}
  Combine Lemma~\ref{lem:x_logx} and~\eqref{eq:bound_log_1_x}.
\end{proof}
%
%%% BEG
%Note that $\barbeta_t := \beta'_t + \lambda S^2 \ge \beta_t $ by noticing that $||\z_{t}||^2_2 - \hth_t^\T\X_{t}^\T \z_{t} $ is nonnegative (obvious from the proof of Corollary~\ref{cor:o2cs}).
%One can carefully examine the definition of $\barbeta_t$ to verify that it is nondecreasing in $t$.
%%% END
Since $r_t$ cannot be bigger than $2LS$,
\begin{equation*}\begin{aligned}
    \sum_{t=1}^T r_t 
  &\le \sum_{t=1}^T \min\{ 2LS, 2L \sqrt{\barbeta_t}||\x_t||_{\ibarV_t}  \}
\\&\le  2L\sqrt{\barbeta_T}\sum_{t=1}^T \min\{ {S}/\sqrt{\barbeta_T}, ||\x_t||_{\ibarV_t}  \} 
\\&\stackrel{\text{(C.-S.)}}{\le} 2L \sqrt{\barbeta_T}\sqrt{T \sum_{t=1}^T \min\lt\{ \fr{S^2}{\barbeta_T}, ||\x_t||^2_{\ibarV_t}  \rt\} }
\\&\stackrel{(\text{Lem.~\ref{lem:regret_useful}})}{\le}  2L\sqrt{\barbeta_T}\sqrt{T \max\{2,S^2/\barbeta_T\} d\log(1+T/(d\lam))  }
\\&= O\lt( L\sqrt{\barbeta_T}\sqrt{T}\cdot \sqrt{d\log T} \rt)
%\\&\stackrel{\eqref{eq:beta_ONS}}{=} \hat O\lt(\fr{L(L + R)}{\kap}\sqrt{d\log^2(T)}\rt)\cdot\sqrt{T}\cdot O\lt(\sqrt{d\log T}\rt) \;,
\end{aligned}\end{equation*}
where C.-S. stands for the Cauchy-Schwartz inequality.
%%% BEG
%This concludes the proof of the first statement in the theorem.
% 
% For the second statemt, simply observe that the first statement of Theorem~\ref{eq:glm} can be directly used in the last equality above by bounding $\max_{t\le T} |\eta_t|$ by $\hat R$.
% This way, we reduce the dependence on $\log T$ by a power of half.
%%% END
\end{proof}

%%%%%%%%%%%%%%%%%%%%%%%%%%%%%%%%%%%%%%%%%%%%%%%%%%%%%%%%%%%%%%%%%%%%%%%%%%%%%%%%
\vspace{-5pt}
\section{Proof of Theorem~\ref{thm:ons}}
\vspace{-5pt}
%%%%%%%%%%%%%%%%%%%%%%%%%%%%%%%%%%%%%%%%%%%%%%%%%%%%%%%%%%%%%%%%%%%%%%%%%%%%%%%%
\begin{proof}
  We closely follow the proof of~\citet{hazan07logarithmic}.
  Since $\ell(z,y)$ is $\kap$-strongly convex w.r.t. $z \in \cB_1(S)$,
  \begin{equation}\label{eq:thm_ons_0}
    \ell(\x_s^\T \th_s, y_s) - \ell(\x_s^\T\th^*,y_s) \le \ell'(\x_s^\T\th_s, y_s) \cdot \x_s^\T(\th_s - \th^*) - \fr{\kap}{2} (\x_s^\T(\th_s-\th^*))^2  \;.
  \end{equation}
  Define $g_s := \ell'(\x_s^\T\th_s, y_s)$.
  Note that by the update rule of Algorithm~\ref{alg:ons},
  \begin{align}    
    \th'_{s+1} - \th^* &= \th_s - \th^* - \fr{g_s}{\kap}\A_s^{-1}\x_s \notag
\\\implies ||\th'_{s+1} - \th^* ||^2_{\A_s} &= ||\th_s - \th^*||^2_{\A_s} - \fr{2g_s}{\kap}\x_s^\T(\th_s-\th^*) + \fr{g_s^2}{\kap^2}||\x_s||^2_{\A_s^{-1}}  \label{eq:thm_ons_1}
  \end{align}
  By the property of the generalized projection (see \citet[Lemma 8]{hazan07logarithmic})
  %The following is known to be true (see~\citet[Lemma 8]{hazan07logarithmic} for a proof):
  \[
    ||\th'_{s+1} - \th^*||^2_{\A_s} \ge || \th_{s+1} - \th^* ||^2_{\A_s} \;.
  \]
  Now, together with~\eqref{eq:thm_ons_1}, 
  \begin{equation*}\begin{aligned}
   ||\th_{s+1} - \th^* ||^2_{\A_s} &\le ||\th_s - \th^*||^2_{\A_s} - \fr{2g_s}{\kap}\x_s^\T(\th_s-\th^*) + \fr{g_s^2}{\kap^2}||\x_s||^2_{\A_s^{-1}}
  \\\implies \sum_{s=1}^t  g_s\x_s^\T(\th_s-\th^*) &\le \sum_{s=1}^t\fr{g_s^2}{2\kap}||\x_s||^2_{\A_s^{-1}} + \fr{\kap}{2} \underbrace{ \sum_{s=1}^t ||\th_s - \th^*||^2_{\A_s} - ||\th_{s+1} - \th^* ||^2_{\A_s} }_{=:D_1}  \;.
  \end{aligned}\end{equation*}
  Note 
  \begin{equation*}\begin{aligned}
  D_1  &= ||\th_1 - \th^* ||^2_{\A_1} + \lt(\sum_{s=2}^t || \th_s - \th^* ||^2_{\A_s} - \sum_{s=2}^t ||\th_s - \th^*||^2_{\A_{s-1}}\rt) - ||\th_{t+1} - \th^*||^2_{\A_{t}} 
  \\&\le ||\th_1 - \th^* ||^2_{\A_1} + \lt(\sum_{s=2}^t || \th_s - \th^* ||^2_{\blue{\A_s}} - \sum_{s=2}^t ||\th_s - \th^*||^2_{\blue{\A_{s-1}}}\rt) 
  \\&\stackrel{(a)}{=} ||\th_1 - \th^* ||^2_{\A_1} + \lt(- ||\th_1-\th^*||^2_{\x_1\x_1^\T} + \sum_{s=1}^t ||\th_s-\th^*||^2_{\x_s\x_s^\T}\rt)
  \\&=  ||\th_1 - \th^* ||^2_{\eps\I} + \sum_{s=1}^t ||\th_s-\th^*||^2_{\x_s\x_s^\T}
     \le  4\eps S^2 + \sum_{s=1}^t ||\th_s-\th^*||^2_{\x_s\x_s^\T}
  \end{aligned}\end{equation*}
  where $(a)$ is due to $\A_s - \A_{s-1} = \x_s\x_s^\T$.
  Therefore,
  \begin{equation*}\begin{aligned}
    \sum_{s=1}^t  g_s\x_s^\T(\th_s-\th^*) 
  &\le \sum_{s=1}^t\fr{g_s^2}{2\kap}||\x_s||^2_{\A_s^{-1}} + 2\eps \kap S^2 + \fr{\kap}{2}\sum_{s=1}^t ||\th_s - \th^*||^2_{\x_s\x_s^\T}  \;.
  \end{aligned}\end{equation*}
  Move the rightmost sum in the RHS to the LHS to see that the LHS now coincide with the RHS of~\eqref{eq:thm_ons_0}.
  This leads to
  \begin{equation*}\begin{aligned}
    \sum_{s=1}^t \ell(\x_s^\T \th_s, y_s) - \ell(\x_s^\T\th^*, y_s) \le \fr{1}{2\kap} \sum_{s=1}^t g_s^2||\x_s||^2_{\A_s^{-1}} + 2\eps \kap S^2 = B^\ONS_t \;.
  \end{aligned}\end{equation*}
  This yields the statement of the theorem.
  %Noting that $g_s = y_s - \mu(\x_s^\T\th^*) = \eta_s$, we conclude the proof for the first statement of the theorem.

  For characterizing the order of $B^\ONS_t$, notice that $g_s$ is a random variable:
  \begin{equation*}\begin{aligned}
    g_s^2 
    &=   ( -y_s + \mu(\x_s^\T\th_s) )^2 = ( -\mu(\x_s^\T\th^*) - \eta_s + \mu(\x_s^\T\th_s) )^2
  \\&\le 2(\mu(\x_s^\T\th_s) - \mu(\x_s^\T\th^*))^2 + 2\eta_s^2
  \\&\le 2(L\cdot\x_s^\T(\th_s - \th^*))^2 + 2\eta_s^2
  \\&\le 2L^2\cdot 4S^2 + 2\eta_s^2
  \end{aligned}\end{equation*}
  Let $\dt<1$ be the target failure rate.
  By the sub-Gaussianity of $\eta_s$, 
  \[
    \P\lt( \forall s\ge1, |\eta_s|^2 \ge 2R^2\log(4s^2/\dt)\rt) \le \sum_{s\ge1} \P\lt(|\eta_s|^2 \ge 2R^2\log(4s^2/\dt)\rt) \le \sum_{s\ge1} \dt/(2s^2) \le \dt \;.
  \]
  Thus, w.p. at least $1-\dt$, $\max_{s\le t} g_s^2 \le 8L^2S^2 + 4R^2\log(4 t^2/\dt) = O(L^2+R^2\log(t/\dt))$.
  Furthermore, $\sum_{s=1}^t ||\x_s||^2_{\A_s^{-1}} \le d\log ( 1+ (t/\eps)) $ by~\citet[Lemma 11]{hazan07logarithmic}.
  Thus,  w.p. at least $1-\dt$, $\forall t\ge1, B^\ONS_t = O\lt(\fr{L^2 + R^2\log(t/\dt)}{\kap}d\log t\rt)$.

  For the case where $|\eta_s|$ is bounded by $\bar R$ w.p. 1 (e.g., $\bar R=\fr{1}{2}$ for Bernoulli), $\max_{s\le t}g_s^2 \le 8L^2 S^2 + 2 \bar{R}^2 $, which leads to $B^\ONS_t = O\lt(\fr{L^2 + {\bar R}^2}{\kap}d\log t\rt)$. 
\end{proof}

%%%%%%%%%%%%%%%%%%%%%%%%%%%%%%%%%%%%%%%%%%%%%%%%%%%%%%%%%%%%%%%%%%%%%%%%%%%%%%%%
\vspace{-5pt}
\section{Proof of Corollaries~\ref{cor:cset_ONS} and ~\ref{cor:regret_glocon_ons}}
\vspace{-5pt}
%%%%%%%%%%%%%%%%%%%%%%%%%%%%%%%%%%%%%%%%%%%%%%%%%%%%%%%%%%%%%%%%%%%%%%%%%%%%%%%%

% Simply use the order of $B^\ONS_t$ identified at the end of the proof of Theorem~\ref{thm:ons} to conclude the proof of Corollary~\ref{cor:cset_ONS}.
The proof of Corollary~\ref{cor:cset_ONS} a trivial consequence of combining Theorem~\ref{thm:o2cs} and Theorem~\ref{thm:ons}.

Corollary~\ref{cor:regret_glocon_ons} is simply a combination of Theorem~\ref{thm:regret_o2cs} and Corollary~\ref{cor:cset_ONS}.
Note that $\barbeta^\ONS_t = \alpha(B^\ONS_t) + \lambda S^2 \ge \beta^\ONS_t $ by noticing that $||\z_{t}||^2_2 - \hth_t^\T\X_{t}^\T \z_{t} $ is nonnegative (from the proof of Theorem~\ref{thm:regret_o2cs}).
This concludes the proof.

%%%%%%%%%%%%%%%%%%%%%%%%%%%%%%%%%%%%%%%%%%%%%%%%%%%%%%%%%%%%%%%%%%%%%%%%%%%%%%%%
\vspace{-5pt}
\section{A Tighter Confidence Set}
\label{sec:supp_tighter}
\vspace{-5pt}
%%%%%%%%%%%%%%%%%%%%%%%%%%%%%%%%%%%%%%%%%%%%%%%%%%%%%%%%%%%%%%%%%%%%%%%%%%%%%%%%

While the confidence set constructed by Theorem~\ref{thm:regret_o2cs} is generic and allows us to rely on any online learner with a known regret bound, one can find a tighter confidence set by analyzing the online learner directly.
We show one instance of such for ONS.
A distinctive characteristic of our new confidence set, denoted by $C^{\ONS_+}_t$, is that it now depends on $y_t$ (note that $C^{\ONS}_t$ depends on $y_1,\ldots,y_{t-1}$ only).

We deviate from the proof of Theorem~\ref{thm:ons}.
Recall that
\begin{equation*}\begin{aligned}
D_1  &= ||\th_1 - \th^* ||^2_{\A_1} + \lt(\sum_{s=2}^t || \th_s - \th^* ||^2_{\A_s} - \sum_{s=2}^t ||\th_s - \th^*||^2_{\A_{s-1}}\rt) - ||\th_{t+1} - \th^*||^2_{\A_{t}} 
\end{aligned}\end{equation*}
We previously dropped the term $||\th_{t+1} - \th^*||^2_{\A_t}$.
We we now keep it, which leads to:
\begin{equation*}\begin{aligned}
  D_1 &\le  4\eps S^2 + \sum_{s=1}^t ||\th_s-\th^*||^2_{\x_s\x_s^\T} - ||\th_{t+1} - \th^*||^2_{\A_{t}}  
\end{aligned}\end{equation*}
Following the same argument,
\begin{equation*}\begin{aligned}
    \lt(\sum_{s=1}^t \ell_s(\x_s^\T \th_t) - \ell_s(\x_s^\T\th^*) \rt)  + \fr{\kap}{2} ||\th_{t+1} - \th^*||^2_{\A_{t}} 
   &\le \sum_{s=1}^t\fr{g_s^2}{2\kap}||\x_s||^2_{\A_s^{-1}} + 2\eps \kap S^2
    = B_t^{\ONS}
\\\implies  \lt(\sum_{s=1}^t \ell_s(\x_s^\T \th_t) - \ell_s(\x_s^\T\th^*) \rt)   
   &\le B_t^{\ONS} - \fr{\kap}{2} ||\th_{t+1} - \th^*||^2_{\A_{t}}\;,
\end{aligned}\end{equation*}

Combining the above with the proof of Theorem~\ref{thm:o2cs},
\begin{equation*}\begin{aligned}
&    \sum_{s=1}^t(\x_s^\T(\th_s-\th^*))^2 
\\&\le 1 + \fr{4}{\kappa}(B^{\ONS}_t - \fr{\kap}{2} ||\th_{t+1} - \th^*||^2_{\A_t} ) + \fr{8R^2}{\kappa^2} \log\lt(\fr{2}{\dt}\sqrt{ 1+ \fr{2}{\kappa}B^{\ONS}_t + \fr{4R^4}{\kappa^4\dt^2} } \rt)
\\&\iff\sum_{s=1}^t(\x_s^\T(\th_s-\th^*))^2 + 2||\th_{t+1} - \th^*||^2_{\A_t} 
      \le 1 + \fr{4}{\kappa}B^{\ONS}_t + \fr{8R^2}{\kappa^2} \log\lt(\fr{2}{\dt}\sqrt{ 1+ \fr{2}{\kappa}B^{\ONS}_t + \fr{4R^4}{\kappa^4\dt^2} } \rt)
\end{aligned}\end{equation*} %\fr{64R^4}{\kappa^4\cdot 4\dt^2}
Define $\z_t = [ \x_s^\T\th_s ]_{s\in[t]}$, $\z^*_t = [\x_s^\T\th^* ]_{s\in[t]}$ and $\z'_s = [\x_s^\T\th_{t+1}]_{s\in[t]}$.
Then, the LHS above is
\begin{equation*}\begin{aligned}
  &||\z_t - \z^*_t||_2^2 + 2||\z'_t - \z^*_t||_2^2 + 2||\th_{t+1}-\th^*||^2_{\eps\I}
\\& = 3 \lnorm \fr{\z_t + 2\z'_t}{3} - \z^*_t \rnorm^2_2 - \fr{1}{3} ||\z_t + 2\z'_t||^2_2 + ||\z_t||^2_2 + 2||\z'_t||^2_2 + 2||\th_{t+1}-\th^*||^2_{\eps\I}
\\& = 3 \lnorm \fr{\z_t + 2\z'_t}{3} - \z^*_t \rnorm^2_2 + \fr{2}{3}||\z_t-\z'_t||^2_2 + 2||\th_{t+1}-\th^*||^2_{\eps\I}\;.
\end{aligned}\end{equation*}
Let $\bar\z_s = (\z_s+2\z'_s)/3$. Then,
\begin{equation*}\begin{aligned}
  \lnorm \bar\z_t - \z^*_t\rnorm^2_2 + ||\th_{t+1}-\th^*||^2_{(2\eps/3)\I} 
%  &\le ||\bar\z_t||^2_2 - \fr{2}{9}||\z_t||^2_2 - \fr{2}{9}||\z'_t||^2_2 + \fr{1}{3} + \fr{4}{r3\kap}B^{\ONS}_t 
  &\le -\fr{2}{9}||\z_t - \z'_t||^2_2 + \fr{1}{3} + \fr{4}{3\kap}B^{\ONS}_t 
  \\&\quad + \fr{8R^2}{3\kap^2}\log\lt(\fr{2}{\dt}\sqrt{ 1+ \fr{2}{\kappa}B^{\ONS}_t + \fr{4R^4}{\kappa^4\dt^2} } \rt)
\end{aligned}\end{equation*}
We lower bound the LHS with $\lnorm \bar\z_t - \z^*_t\rnorm^2_2 + \fr{2\eps}{3}||\th^*||^2_2 - \fr{4\eps}{3} S^2 =: D_2 $.
Define $\barW_t := \X_t^\T\X_t + (2\eps/3)\I$ and $\hth^+_t = \ibarW_t\X_t^\T\bar\z_t$.
Then,
\begin{equation*}\begin{aligned}
D_2 
  &= ||\th^* - \hth^+_t||^2_{\barW_t} + \underbrace{||\bar\z_t - \X_t \hth^+_t||^2_2 + \fr{2\eps}{3}||\hth^+_t||^2_2}_{ = ||\bar\z_t||^2_2 - \bar\z_t^\T\X_t\hth^+_t } -\fr{4\eps}{3} S^2 \;.
\end{aligned}\end{equation*}
Thus,
\begin{equation*}\begin{aligned}
||\th^* - \hth^+_t||^2_{\barW_t} 
  &\le -||\bar\z_t||^2_2 + \bar\z_t^\T\X_t\hth^+_t + \fr{4\eps}{3} S^2 -\fr{2}{9}||\z_t - \z'_t||^2_2 + \fr{1}{3} + \fr{4}{3\kap}B^{\ONS}_t 
  \\&\quad + \fr{8R^2}{3\kap^2}\log\lt(\fr{2}{\dt}\sqrt{ 1+ \fr{2}{\kappa}B^{\ONS}_t + \fr{4R^4}{\kappa^4\dt^2} } \rt) %...+ ||\bar\z_t||^2_2 + \fr{1}{2} + \fr{1}{\kap}B^{\ONS}_t 
% \\&\quad + \fr{4R^2}{\kap^2}\log\lt(\fr{2}{\dt}\sqrt{ 1+ \fr{2}{\kappa}B^{\ONS}_t + \fr{4R^4}{\kappa^4\dt^2} } \rt)         
% \\&\le \bar\z_t^\T\X_t\hth^+_t +\eps S^2 + \fr{1}{2} + \fr{1}{\kap}B^{\ONS}_t 
% \\&\quad + \fr{4R^2}{\kap^2}\log\lt(\fr{2}{\dt}\sqrt{ 1+ \fr{2}{\kappa}B^{\ONS}_t + \fr{4R^4}{\kappa^4\dt^2} } \rt)         
\\&=:\beta^{\ONS_+}_t \;.
\end{aligned}\end{equation*}
This leads to the following confidence set:
\begin{equation*}
  C^{\ONS_+}_t = \{ \th\in\dsR^d: ||\th - \hth^+_t||^2_{\barW_t} \le \beta^{\ONS_+}_t  \} \;.
\end{equation*}

%%%%%%%%%%%%%%%%%%%%%%%%%%%%%%%%%%%%%%%%%%%%%%%%%%%%%%%%%%%%%%%%%%%%%%%%%%%%%%%%
\vspace{-5pt}
\section{Details on GLOC-TS}
\vspace{-5pt}
%%%%%%%%%%%%%%%%%%%%%%%%%%%%%%%%%%%%%%%%%%%%%%%%%%%%%%%%%%%%%%%%%%%%%%%%%%%%%%%%

For simplicity, we present the algorithm and the analysis when GLOC-TS is combined with the confidence set $C^{\ONS}_t$.
To present the algorithm, we use the definition $\beta_t^{\ONS}$ from Corollary~\ref{cor:cset_ONS}. % (recall that it is defined through $B^\ONS_t$ defined in Theorem~\ref{thm:ons}).
We use the notation $\beta_{t}^{\ONS}(\dt)$ to show the dependence on $\dt$ explicitly.
We present the algorithm of GLOC-TS in Algorithm~\ref{alg:gloc_ts}.
\begin{algorithm}[h]
  \begin{algorithmic}[1]
    \STATE \textbf{Input}: time horizon $T$, $\dt \in(0,1)$, $\lam>0$, $S>0$ , $\kap>0$.
    \STATE Let $\dt' = \dt/(8T)$.
    \FOR {$t=1,2,\ldots,T$}
%      \STATE Receive a prediction $\th_t$ from $\cA$. %%% FIXME I removed it since it raises a question...
      \STATE Sample $\bfxi_t \sim \cN(0, \I)$.
      \STATE Compute parameter $\dot\th_t = \hth_{t-1} + \sqrt{\beta^{\ONS}_{t-1}(\dt')}\cdot\barV_{t-1}^{-1/2} \bfxi_t $
      \STATE Solve $\x_t = \arg \max_{\x\in\cX} \x^\T\dot\th_t$.
      \STATE Pull $\x_t$ and then observe $y_t$.
      \STATE Feed into $\cB$ the loss function $\ell_t(\th) = \ell(\x_t^\T \th, y_t)$.
      \STATE Update $\hth_t$ and $\barV_t$.
    \ENDFOR
  \end{algorithmic}
  \caption{GLOC-TS (GLOC - Thompson Sampling)}
  \label{alg:gloc_ts}
\end{algorithm}

Define $\gam_t(\dt) = \beta_t(\dt) 2 d \log(2d/\dt)$ and $p=\fr{1}{4\sqrt{e\pi}}$.
We present the full statement of Theorem~\ref{thm:gloc-ts} as follows.

\paragraph{Theorem~\ref{thm:gloc-ts}}
Let $\dt' = \dt/(8T)$.
The cumulative regret of GLOC-TS over $T$ steps is bounded as, w.p. at least $1-\dt$,
\begin{equation*}\begin{aligned}
    \text{Regret}_T &\le L\lt(\sqrt{\beta_T(\dt')} + \sqrt{\gam_T(\dt')}(1+2/p)\rt)\sqrt{2Td\log(1+T/\lam)} + \fr{2L}{p} \sqrt{\gam_T(\dt')} \sqrt{\fr{8T}{\lam}\log(4/\dt)} \\
  \\&= O\lt(\fr{L(L + R)}{\kap} d^{3/2} \sqrt{\log(d) T} \log^{3/2} T \rt)
\end{aligned}\end{equation*}
\begin{proof}
  Note that the proof is a matter of realizing that the proof of~\citet[Lemma 4]{abeille17linear} relies on a given confidence set $C_t$ in a form of~\eqref{eq:cset}.
  To be specific, notice that the notations $k_\mu$, $c_\mu$, and $\beta_t(\dt')$ used in~\citet{abeille17linear} is equivalent to $L$, $\kap$, and $\kap\sqrt{\beta_t(\dt)}$ in this paper, respectively.
  Furthermore, our result stated in~\eqref{eq:cset_ONS} can replace the result of~\citet[Prop. 11]{abeille17linear}.
  Finally, one can verify that the proof of \citet[Lemma 4]{abeille17linear} leads to the proof of the theorem above.
\end{proof}
Although the theorem above is stated under the fixed-buget setting, one can easily change the failure rate $\dt'$ to $O(t^2/\dt)$ and enjoy an anytime regret bound.

%%%%%%%%%%%%%%%%%%%%%%%%%%%%%%%%%%%%%%%%%%%%%%%%%%%%%%%%%%%%%%%%%%%%%%%%%%%%%%%%
\vspace{-5pt}
\section{QGLOC without \texorpdfstring{$r$}{}}
\vspace{-5pt}
%%%%%%%%%%%%%%%%%%%%%%%%%%%%%%%%%%%%%%%%%%%%%%%%%%%%%%%%%%%%%%%%%%%%%%%%%%%%%%%%

We present an alternative definition of $m_t$ that does not rely on $r$: %  defined in~\eqref{def-m}:
\begin{align*}
  \barm'_{t-1} = ||\x^{\text{Greedy}}||_{\ibarV_{t-1}} \text{, where } \x^{\text{Greedy}} := \arg \max_{\x\in\cX_{t-1}} \la \hatth_{t-1}, \x \ra 
\end{align*}
We claim that $\barm'_{t-1} \le ||\x_t^{\QGLOC}||_{\ibarV_{t-1}}$.
To see this, suppose not: $\barm'_{t-1} > ||\x_t^{\QGLOC}||_{\ibarV_{t-1}}$.
Then, $\x^{\text{Greedy}} \neq \x^{\QGLOC}$.
Furthermore, $\x^{\text{Greedy}}$ must be the maximizer of the QGLOC objective function while $\x^{\text{Greedy}} \neq \x^{\QGLOC}$, which is a contradiction.  
The claim above allows all the proofs of our theorems on QGLOC to go through with $\barm'_{t-1}$ in place of $\barm_{t-1}$.

However, computing $\barm'_{t-1}$ requires another hashing for finding the greedy arm defined above. 
Although this introduces a factor of 2 in the time complexity, it is quite cumbersome in practice, which is why we stick to $m_{t-1}$ defined in~\eqref{def-m} in the main text.

%%%%%%%%%%%%%%%%%%%%%%%%%%%%%%%%%%%%%%%%%%%%%%%%%%%%%%%%%%%%%%%%%%%%%%%%%%%%%%%%
\vspace{-5pt}
\section{Proof of Theorem~\ref{thm:x_t9}}
\vspace{-5pt}
\label{sec:proof-thm-x_t9}
%%%%%%%%%%%%%%%%%%%%%%%%%%%%%%%%%%%%%%%%%%%%%%%%%%%%%%%%%%%%%%%%%%%%%%%%%%%%%%%%

Let us first present the background.
Denote by $x^\GLOC_t$ the solution of the optimization problem at line 3 of Algorithm~\ref{alg:gloc}.
To find $x^\GLOC_t$, one can fix $\x$ and find the maximizer $\tilth(\x) := \max_{\th\in C_{t-1}} \x^\T\th$ in a closed form using the Lagrangian method:
\begin{align} \label{eq:glocopt_tilth}
  \tilth(\x) = \hatth_{t-1} + \sqrt{\beta_{t-1}} \cdot\fr{ \barV^{-1}_{t-1} \x }{ ||\x||_{\barV^{-1}_{t-1}} } \;,
\end{align}
which is how we obtained~\eqref{eq:glocopt_ext}.

The following lemma shows that the objective function~\eqref{eq:qglocopt} plus $c_0\beta^{3/4} \barm_{t-1}$ is an upper bound of the GLOC's objective function~\eqref{eq:glocopt_ext}.
Note that this holds for any sequence $\{\beta_t\}$.
\begin{lem}\label{lem:ub}
  \begin{equation*}\begin{aligned}
  &\la \hatth_{t-1}, \x \ra + \sqrt{\beta_{t-1}}||\x||_{\ibarV_{t-1}} \\
  &\le \la \hatth_{t-1}, \x \ra  + \fr{\beta_{t-1}^{1/4}}{4 c_0 \barm_{t-1}} \cdot || \x ||^2_{\barV^{-1}_{t-1}}  + c_0 \beta_{t-1}^{3/4} \barm_{t-1} \;.
  \end{aligned}\end{equation*}
  Furthermore,
  \emph{
  \begin{equation}\begin{aligned}
    &\la \hatth_{t-1}, \x_{t}^\GLOC \ra + \sqrt{\beta_{t-1}} ||\x_{t}^\GLOC||_{\barV^{-1}_{t-1}} \\
    &\le \la \hatth_{t-1}, \x_{t}^\QGLOC \ra + \fr{\beta_{t-1}^{1/4}}{4 c_0 \barm_{t-1}}  ||\x_{t}^\QGLOC||^2_{\barV^{-1}_{t-1}}  + c_0 \beta_{t-1}^{3/4} \barm_{t-1}  \;.
  \end{aligned}\end{equation}
  }
\end{lem}
\begin{proof}
  Recall the GLOC optimization problem defined at line 3 of Algorithm~\ref{alg:gloc}.
  For a fixed $\x$, we need to solve
  \begin{equation*}\begin{aligned}
    \tilth(\x) = \arg \min_{\th} &\quad -\la \th, \x \ra \\
    \mbox{s.t.}        &\quad || \th - \hatth_{t-1} ||^2_{\barV_{t-1}} - \beta_{t-1} \le 0
  \end{aligned}\end{equation*}
  The Lagrangian is $\mathcal{L}(\th,\tau) := -\la \th, \x \ra + \tau(|| \th - \hatth_{t-1} ||^2_{\barV_{t-1}} - \beta_{t-1})$, and thus we need to solve
  \[
  \max_{\tau \ge 0} \min_{\th} \quad \mathcal{L}(\th,\tau) \;.
  \]
  According to the first-order optimality condition,
  \begin{align}
    -\fr{\partial \mathcal{L}(\th,\tau)}{\partial \th} &= \x - \tau(2\barV_{t-1} \th - 2\barV_{t-1} \hatth_{t-1}) = 0 \notag
  \\\th &= \hatth_{t-1} + (2\tau)^{-1} \barV^{-1}_{t-1}\x \;, \label{eq:th_of_tau}
  \end{align}
  which results in $ \lt( \min_{\th} \mathcal{L}(\th,\tau) \rt) = -\la \hatth_{t-1}, \x \ra - (4\tau)^{-1} || \x ||^2_{\barV^{-1}_{t-1}} - \tau \beta_{t-1}$.
  It remains to solve
  \[
  \max_{\tau\ge 0} \quad -\la \hat\th, \x \ra - (4\tau)^{-1} || \x ||^2_{\barV^{-1}_{t-1}} - \tau \beta_{t-1} ,
  \]
  whose first-order optimality says that the solution is $\tau_* := (2\sqrt{\beta_{t-1}})^{-1} ||\x||_{\barV^{-1}}$.
  Plugging $\tau \larrow \tau_*$ in~\eqref{eq:th_of_tau} leads to the solution $\tilth(\x)$ defined in~\eqref{eq:glocopt_tilth}. %, but then we will have square root around $\x^\T \ibarV_{t-1}\x$, which is undesirable.
  Note that any choice of $\tau \ge 0$ leads to a lower bound on the Lagrangian $\calL(\tilth(\x),\tau)$ .
  Define 
  \begin{align*}
  \tilde\tau := c_0 \beta_{t-1}^{-1/4} \barm_{t-1} \;.
  \end{align*}
  Then, 
  \begin{align*}
  \calL(\tilth(\x), \tilde\tau) 
  &= -\la \hatth_{t-1}, \x \ra - \fr{\beta_{t-1}^{1/4}}{4 c_0 \barm_{t-1}}  ||\x||^2_{\barV^{-1}_{t-1}}  - c_0 \beta_{t-1}^{3/4} \barm_{t-1} \\
  &\le \calL(\tilth(\x), \tau_*) = -\la \hatth_{t-1}, \x \ra - \sqrt{\beta_{t-1}} ||\x||_{\barV^{-1}_{t-1}} \;.
  \end{align*}
  This concludes the first part of the lemma.
  The second part of the lemma trivially follows from the first part by the definition~\eqref{eq:qglocopt}.
\end{proof}  

%Define $\bfeta_t = [\eta_1;\cdots;\eta_t]$.
% Note that
% \begin{align}
% E_1(\dt) \implies ||\X_{t}^\T \bfeta_{t} ||_{\ibarV_{t}} \le R\sqrt{d \log\lt( \fr{1+t/(d\lam)}{\dt}\rt)} \;, \label{bound_X_eta} \\
% \sum_{t=1}^T \log\lt(1+ ||\x_t||^2_{\ibarV_{t-1}}\rt) \le 2d \log\lt( 1+T/(d\lam)\rt) \;, \label{bound_log_1_x}
% \end{align}
% which are due to Theorem 1 and Lemma 11 of~\cite{ay11improved} respectively. 

% Using Lemma~\ref{lem:x_logx}, 
% We now prove Theorem~\ref{thm:x_t9} below.
% The key step of the proof is that the maximum of the QGLOC objective function~\eqref{eq:qglocopt} plus $c_0\beta^{3/4} \barm_{t-1}$ is a tight upper bound of the maximum of the GLOC objective function~\eqref{eq:glocopt_ext}.
\begin{proof}
Assume $\th^* \in C^{\ONS}_t$ for all $t\ge1$, which happens w.p. at least $1-\dt$.
Suppose we pull $\x_t^\QGLOC$ at every iteration.
While we use the confidence set $C^{\ONS}_t$, we omit the superscript $\ONS$ from $\beta^{\ONS}_t$ for brevity.
Recall that $\bar\beta_t$ is an upper bound on $\beta_t$ that is nondecreasing in $t$.
We bound the instantaneous regret $r_t$ as follows:
\begin{align*}
  \fr{r_t}{L}
&= \fr{1}{L} \lt(\mu(\la \th^*, \x_{t,*} \ra) - \mu(\la \th^*, \x_t^\QGLOC \ra)\rt) \\
&\le (\la \th^*, \x_{t,*} \ra - \la \th^*, \x_t^\QGLOC \ra) \\
&\le \la \tilth(\x_t^\GLOC), \x_{t}^\GLOC \ra - \la \th^*, \x_t^\QGLOC \ra \\
&= \la \hatth_{t-1}, \x_{t}^\GLOC \ra + \sqrt{\bar\beta_{t-1}} ||\x_{t}^\GLOC||_{\barV^{-1}_{t-1}} - \la \th^*, \x_t^\QGLOC \ra \\
&\stackrel{\text{(Lem.~\ref{lem:ub})}}{\le} \la \hatth_{t-1}, \x_t^\QGLOC \ra + \fr{\bar\beta_{t-1}^{1/4}}{4 c_0 \barm_{t-1}} { ||\x_t^\QGLOC||^2_{\barV^{-1}_{t-1}} } + c_0 \bar\beta_{t-1}^{3/4} \barm_{t-1}  - \la \th^*, \x_t^\QGLOC \ra \\
&= \la \hatth_{t-1} - \th^*, \x_t^\QGLOC \ra + \fr{\bar\beta_{t-1}^{1/4}}{4 c_0 \barm_{t-1}} { ||\x_t^\QGLOC||^2_{\barV^{-1}_{t-1}} } + c_0 \bar\beta_{t-1}^{3/4} \barm_{t-1}  \\
&\le \underbrace{||\hatth_{t-1} - \th^*||_{\barV_{t-1}} ||\x_t^\QGLOC||_{\barV^{-1}_{t-1}}}_{=: A_1(t)} + \underbrace{\fr{\bar\beta_{t-1}^{1/4}}{4 c_0 \barm_{t-1}} { ||\x_t^\QGLOC||^2_{\barV^{-1}_{t-1}} }}_{:= A_2(t)} + \underbrace{c_0 \bar\beta_{t-1}^{3/4} \barm_{t-1}}_{=: A_3(t)} \;.
\end{align*}
Note that $\la \th^*,\x \ra \in [-S,S]$ implies that $r_t \le 2LS$.
Then,
\begin{align*}
  r_t &\le \min\{2LS, L\cdot A_1(t) + L\cdot A_2(t) + L\cdot A_3(t) \} \\
      &\le L\min\{2S, A_1(t)\} + \min\{2S, A_2(t)\} + \min\{2S, A_3(t)\} \;.
\end{align*}
where the last inequality can be shown by a case-by-case analysis on each $\min$ operator.

Now, we consider computing $\sum_{t=1}^T r_t$.
Using the same argument as the proof of Theorem~\ref{thm:regret_o2cs} and Corollary~\ref{cor:regret_glocon_ons},
\begin{align*}
 L\sum_{t=1}^T \min\{2S, A_1(t)\} = O\lt(\fr{L(L+R)}{\kap}d\sqrt{T}\log^{3/2}(T)\rt)  \;.
\end{align*}

%%% BEG; IS THIS NEEDED?
% Also, assume $\fr{8 c_0}{\beta^{1/4}_{0} \sqrt{1 + \lambda}} \le 1$ and $\fr{2}{c_0 \beta^{3/4}_{0}} \le 1$.
%%% END

%We use $c_{0,t}$ in place of $c_0$ to be explicit on $t$.
Then,
\begin{align*}
  &L\sum_{t=1}^T \min\lt\{ 2S, \fr{\bar\beta_{T-1}^{1/4}}{4 c_0 \barm_{t-1}} { ||\x_t^\QGLOC||^2_{\barV^{-1}_{t-1}} } \rt\} 
\\&\le L\sum_{t=1}^T \min\lt\{ 2S, \fr{\bar\beta_{T-1}^{1/4}}{4 c_0 r} \sqrt{T + \lambda} ||\x_t^\QGLOC||^2_{\barV^{-1}_{t-1}}  \rt\}
\\&\le \fr{L\bar\beta^{1/4}_{T-1}}{4c_0 r} \sqrt{T + \lambda} \sum_{t=1}^T \min\lt\{\fr{8 c_0 r S }{\bar\beta^{1/4}_{T-1} \sqrt{T + \lambda}}, ||\x_t^\QGLOC||^2_{\barV^{-1}_{t-1}} \rt\}
\\&\stackrel{(\text{Lem.~\ref{lem:x_logx}})}{\le} \fr{L\bar\beta^{1/4}_{T-1}}{4c_0 r} \sqrt{T + \lambda} \max\lt\{2, \fr{8 c_0 r S }{\bar\beta^{1/4}_{T-1} \sqrt{T + \lambda}} \rt\} \sum_{t=1}^T \log( 1 + ||\x_t^\QGLOC||^2_{\barV^{-1}_{t-1}} ) 
%\\&\stackrel{\text{(Cor.~\ref{cor:cset_ONS})},\eqref{eq:bound_log_1_x}}{=} L \cdot O\lt( \lt(\fr{L^2+R^2}{\kap^2}\rt)^{1/4} \log^{1/4}(T) \rt) O\lt( \lt(\fr{L^2+R^2}{\kap^2} d\log^2 T\rt)^{1/4} \rt) \cdot O(\sqrt{T}) \cdot O(d \log T) 
%\\&= O\lt( \fr{ L ( L +  R) }{\kap} d^{5/4} \sqrt{T} \log^{7/4} T \rt) \\
\\&\stackrel{\eqref{eq:bound_log_1_x}}{=} L \fr{1}{c_0} \cdot O\lt( \lt(\fr{L^2+R^2}{\kap^2} d\log^2 T\rt)^{1/4} \rt) \cdot O(\sqrt{T}) \cdot O(d \log T) 
\\&= O\lt( \fr{1}{c_0}L\lt(\fr{ L +  R }{\kap}\rt)^{1/2} d^{5/4} \sqrt{T} \log^{3/2} T \rt) \\
\end{align*}
and 
\begin{align*}
  &L\sum_{t=1}^T \min \{ 2S, c_0 \bar\beta_{t-1}^{3/4} \barm_{t-1} \}
\\&\stackrel{\eqref{def-m}}{\le} L\sum_{t=1}^T \min \{ 2S, c_0 \bar\beta_{T-1}^{3/4} || \x_t^\QGLOC ||_{\barV^{-1}_{t-1}} \} 
\\&\le Lc_0 \bar\beta_{T-1}^{3/4} \sum_{t=1}^T  \min \lt\{ \fr{2S}{c_0\bar\beta^{3/4}_{T-1}},  || \x_t^\QGLOC ||_{\barV^{-1}_{t-1}} \rt\} 
\\&\stackrel{(a)}{\le} Lc_0 \bar\beta_{T-1}^{3/4} \sqrt{T \sum_{t=1}^T  \min \lt\{ \lt(\fr{2S}{c_0\bar\beta^{3/4}_{T-1}}\rt)^2,  || \x_t^\QGLOC ||^2_{\barV^{-1}_{t-1}} \rt\} } 
\\&\stackrel{(\text{Lem.~\ref{lem:x_logx}})}{\le} Lc_0 \bar\beta_{T-1}^{3/4} \sqrt{T \max\lt\{ 2, \lt(\fr{2S}{c_0\bar\beta^{3/4}_{T-1}}\rt)^2 \rt\} \sum_{t=1}^T  \log(1 + || \x_t^\QGLOC ||^2_{\barV^{-1}_{t-1}}) } \\
% \\&\stackrel{\text{(Cor~\ref{cor:cset_ONS})},\eqref{eq:bound_log_1_x}}{=} L\cdot O\lt( \lt(\fr{L^2+R^2}{\kap^2}\log(T)\rt)^{-1/4} \rt)\cdot O\lt( \lt(\fr{L^2+R^2}{\kap^2} d\log^2 T\rt)^{3/4} \rt) \cdot O(\sqrt{T d \log T})
% \\&= O\lt( \fr{ L(L + R)}{\kap} d^{5/4} \sqrt{T} \log^{7/4} T \rt) \; ,
\\&\stackrel{\eqref{eq:bound_log_1_x}}{=} c_0 L \cdot O\lt( \lt(\fr{L^2+R^2}{\kap^2} d\log^2 T\rt)^{3/4} \rt) \cdot O(\sqrt{T d \log T})
\\&= O\lt( c_0 L \lt(\fr{ L + R}{\kap}\rt)^{3/2} d^{5/4} \sqrt{T} \log^2 T \rt) \; ,
\end{align*}
where $(a)$ is due to the Cauchy-Schwartz inequality.
Therefore, the regret is 
\[
  O\lt(  L(\fr{1}{c_0}\lt( \fr{L+R}{\kap} \rt)^{1/2} + c_0\lt( \fr{L+R}{\kap} \rt)^{3/2} )  d^{5/4} \sqrt{T} \log^2(T)\rt)\;.
\]
One can see that setting $c_0 = c'_0 \lt( \fr{L+R}{\kap} \rt)^{-1/2}$ leads to the stated regret bound.
Note that one can improve $\log^2(T)$ in the regret bound to $\log^{7/4}(T)$ by making $c_0$ scale with $\log^{-1/4}(t)$, which is left as an exercise.

When the noise $|\eta_t|$ is bounded, we have a tighter $\beta_t$ and thus we can replace $\log^2(T)$ in the regret bound to $\log^{5/4}(T)$.
\end{proof}

%%%%%%%%%%%%%%%%%%%%%%%%%%%%%%%%%%%%%%%%%%%%%%%%%%%%%%%%%%%%%%%%%%%%%%%%%%%%%%%%
\vspace{-5pt}
\section{On \texorpdfstring{${c_0}$}{} of QGLOC}
\vspace{-5pt}
%%%%%%%%%%%%%%%%%%%%%%%%%%%%%%%%%%%%%%%%%%%%%%%%%%%%%%%%%%%%%%%%%%%%%%%%%%%%%%%%

Observe that in~\eqref{eq:qglocopt} $c_0$  is a free parameter that adjusts the balance between the exploitation (the first term) and exploration (the second term).
This is an interesting characteristic that is not available in existing algorithms but is attractive to practitioners. 
Specifically, in practice existing bandit algorithms like OFUL~\cite{ay11improved}, LTS~\cite{agrawal13thompson}, and others~\cite{filippi10parametric,zhang16online} usually perform exploration more than necessary, so one often enforces more exploitation by multiplying a small constant less than 1 to $\sqrt{\beta_t}$; e.g., see~\cite{yue12hierarchical,chapelle11anempirical}.
Applying such a trick is theoretically not justified and foregoes the regret guarantee for existing algorithms, so a practitioner must take a leap of faith.
In contrast, adjusting $c_0$ of QGLOC is exactly the common heuristic but now does not break the regret guarantee, which can assure practitioners.

%%%%%%%%%%%%%%%%%%%%%%%%%%%%%%%%%%%%%%%%%%%%%%%%%%%%%%%%%%%%%%%%%%%%%%%%%%%%%%%%
\vspace{-5pt}
\section{Details on Hashing}
\label{sec:supp-hashing}
\vspace{-5pt}
%%%%%%%%%%%%%%%%%%%%%%%%%%%%%%%%%%%%%%%%%%%%%%%%%%%%%%%%%%%%%%%%%%%%%%%%%%%%%%%%

We first briefly introduce hashing methods for fast similarity search.
Here, the similarity measure is often Euclidean distance~\cite{datar04locality} or inner product~\cite{shrivastava14asymmetric}.
For a comprehensive review, we refer to~\citet{wang14hashing}.
The hashing methods build a hash table that consists of buckets where each bucket is identified by a unique hash key that is a sequence of $k$ integers.
At the hashing construction time, for each data point $\x$ in the database we compute its hash key $h(\x)$ using a function $h$ (details shown below).
% The hash function outputs $k$ integers, so multiple data points share the same hash key.
We then organize the hash table so that each bucket contains pointers to the actual data points with the same hash key.
The hash functions are decided at the construction time.
Typically, for $d'$-dimensional hashing one draws $k$ independent normally-distributed $d'$-dimensional vectors, which we call \emph{projection vectors}.
The hash function $h(\x)$ outputs a discretized version of the inner product between these $k$ vectors and the data point $\x$.
When processing a query $\q$, we compute the hash key $h(\q)$, retrieve the corresponding bucket, compute the similarities between the query and the data points therein, and pick the most similar one. 
It is important that one uses the same projection vectors for constructing the table and determining the hash key of the query.
This means one needs to store the projection vectors.
Finally, one typically constructs $U$ independent hash tables to reduce the chance of missing very similar points (i.e., to increase the recall).

We now turn to operating QGLOC with hashing.
Note that one would like to use an accurate hashing scheme since how accurately one solves~\eqref{eq:qglocopt} can impact the regret.
However, a more accurate hashing have a higher space and time complexity.
We characterize such a tradeoff in this section.

Recall that we aim to guarantee the accuracy of a MIPS hashing by the $\ccH$-MIPS-ness.
As we have mentioned in the main text, however, existing MIPS algorithms do not directly offer a $\ccH$-MIPS guarantee.
%One way to have a $\ccH$-MIPS guarantee is to construct a series of existing MIPS hashings, which we present here.
%Our construction is similar to~\cite{indyk12approximate} but is much simpler since the maximizer of the QGLOC objective function~\eqref{eq:qglocopt} is bounded w.h.p.
Instead of the $\ccH$-MIPS guarantee, the standard MIPS algorithms provide a less convenient guarantee for an input parameter $c$ and $M$ as follows:
\begin{defn}
 Let $\cX \subseteq \dsR^{d'}$ such that $|\cX| < \infty$.
 A data point $\til\x \in \cX$ is called $(c,M)$-MIPS w.r.t. a given query $\q$ if it satisfies $\la \q, \til\x \ra \ge c M$.
 An algorithm is called $(c,M)$-MIPS if, given a query $\q\in {\dsR}^{d'}$, it retrieves $\x \in \cX$ that is $(c,M)$-MIPS w.r.t. $\q$ whenever there exists $\x' \in \cX$ such that $\la \q, \x' \ra \ge M$.
 \vspace{-4pt}
\end{defn}
Note that when there is no such $\x$ that $\la \q,\x \ra \ge M$, retrieving any arbitrary vector in $\cX$ is qualified as being $(c,M)$-MIPS.
This also means that, by its contrapositive, if the hashing returns a vector that is not $(c,M)$-MIPS w.r.t. $\q$, then there is no $\x$ such that $\la\q,\x\ra \ge M$ with high probability.

We emphasize that, to enjoy the $(c,M)$-MIPS-ness, a hashing must be built based on $c$ and $M$.
If we know $\max_{\x\in\cX_t} \la \q_t, \x \ra$ ahead of time, then we can just set $M = \max_{\x\in\cX_t} \la \q_t, \x \ra$ and build a hashing that is $(c,M)$-MIPS for some $c$, which gives a $c$-MIPS guarantee for the query $\q_t$.
However, one does not have such information, and each query $\q_t$ has its own ``useful'' value $M$.
% However, for each query $\q_t$ there exists a different value of $M$ that is ``useful'', and one does not know ahead of time what $M$ values would be useful for the future queries.

To overcome such a difficulty, it seems natural to construct a $c$-MIPS hashing using multiple $(c,M)$-MIPS hashings with various $M$ values covering a wide range.
Indeed, such a construction is described in~\cite{indyk12approximate} for locality-sensitive hashing based on the Euclidean distance, which is complicated by multiple subroutines.
We present here a streamlined construction of a $\ccH$-MIPS hashing thanks to the existence of a high-probability upper and lower bound on the maximum of the QGLOC objective~\eqref{eq:qglocopt} as we show below.
Note that one can derive a similar result for GLOC-TS.

Hereafter, we omit $\ONS$ from $\beta^\ONS_t$.
Note that there exists a simple upper bound $\barbeta_t$ on $\beta_t$ (see the proof of Corollary~\ref{cor:regret_glocon_ons}). 
 %
%by noticing that $||\z_t||^2_2 - \hth_t^\T\X_t^\T\z_t $ is nonnegative (obvious from the proof of Theorem~\ref{thm:regret_o2cs}):
% \begin{equation*}\begin{aligned}
%   \beta_t \le \barbeta_t := \beta'_t + \lam S^2
% \end{aligned}\end{equation*}
%which is due to~\cite[Lemma 10]{ay11improved}.
Define $E_1(\dt)$ to be the event that $\th^*$ belongs to $C_t$ for all $t\ge1$.
\vspace{-3pt}
\begin{lem}\label{lem:qgloc-bound} 
  Assume $E_1(\dt)$ and $\max_{\x \in \cX_t}\la \th^*, \x \ra \ge 1/2$.
  Suppose the target time horizon $T$ is given.
  Then,
  \vspace{-4pt}
  \begin{equation*}\begin{aligned}
    M_{\min} 
    &:= 1/2 
      \le \max_{t\in[T]} \max_{\x\in\cX_t} \la\hatth_{t-1}, \x\ra + \fr{\beta_{t-1}^{1/4}}{4 c_0 \barm_{t-1}} ||\x||^2_{\ibarV_{t-1}} \\
    & \le \sqrt{d}S + \barbeta_{T-1}^{1/4}\cdot \fr{\sqrt{T+\lam}}{4c_0 r \lam}
      =: M_{\max} \;.
  \end{aligned}\end{equation*}
\end{lem}
\vspace{-4pt}
Before presenting the proof, note that Lemma~\ref{lem:qgloc-bound} assumes that $\max_{\x \in \cX_t}\la \th^*, \x \ra \ge 1/2$. 
In practice, this is not a restrictive assumption since it means that there exists at least one arm for which the reward is decently large in expectation. %; if not, any algorithm would work almost equally badly since all items have small rewards in the first place.
In interactive retrieval systems with binary user feedback, for example, it is reasonable to assume that there exists at least one item to which the user is likely to give a positive reward since otherwise any algorithm would work almost equally badly.
One can change $1/2$ to any reasonable number $v$ for which the reward $\mu(v)$ is considered high. %user's feedback is considered to be positive.
\begin{proof}
  To show the lowerbound, recall that the objective function~\eqref{eq:qglocopt} is derived as an upperbound of the original GLOC objective function~\eqref{eq:glocopt_ext}.
  Since the original GLOC objective function has $\th^*$ as a feasible point by $E_1(\dt)$, $\max_{\x\in\cX_t} \la \th^*,\x\ra$ becomes a trivial lowerbound of the maximum of the GLOC objective and also of the maximum of QGLOC objective~\eqref{eq:qglocopt}.
  This proves the lowerbound of the Lemma.

  For the remaining part of the proof, we use notation $\X$, $\bfeta$, $\z$ and $\barV$ in place of $\X_t$, $\bfeta_t$, $\z_t$, and $\barV_t$ respectively (recall that $\z$ is defined in Section~\ref{sec:gloc}).

  It is easy to see that the eigenvalues of $\X\barV^{-2}\X^\T$ are all less than 1 (use the SVD of $\X$).
  Furthermore, $z_s = \x_s^\T\th_s \in [-S,S]$, and so $||\z||_2 \le \sqrt{d} S$.
  Using these facts, 
  \begin{equation*}\begin{aligned}
    ||\hth_{t-1}||_2 
    &= || \barV^{-1} \X^\T \z||_2
  \\&= \sqrt{(\z)^\T \X \barV^{-2} \X^\T \z}
  \\&= ||(\X\barV^{-2}\X^\T)^{1/2}  \z ||_2
  \\&\le ||\z||_2
     \le \sqrt{d}S \;.
  \end{aligned}\end{equation*}
  Finally, % upperbound the maximum of the objective function~\eqref{eq:qglocopt} plus $c_0\beta^{3/4} \barm_{t-1}$:
  \begin{equation*}\begin{aligned}
  \max_{t\in[T]} \max_{\x\in\cX_t} \la\hatth_{t-1}, \x\ra + \fr{\beta_{t-1}^{1/4}}{4 c_0 \barm_{t-1}} ||\x||^2_{\ibarV_{t-1}} 
    &\le \max_{t\in[T]} \max_{\x\in\cX_t} ||\hatth_{t-1}||_2 + \fr{\beta_{t-1}^{1/4}}{4 c_0 \barm_{t-1}} ||\x||^2_{\ibarV_{t-1}}
  \\&\le \max_{t\in[T]} \sqrt{d}S
     + \beta_{t-1}^{1/4} \cdot \fr{\sqrt{t + \lambda}}{ 4 c_0 r} \cdot \fr{1}{\lambda} 
  \\&\le \sqrt{d} S + \barbeta_{T-1}^{1/4} \cdot \fr{\sqrt{T + \lambda}}{ 4 c_0 r } \cdot \fr{1}{\lambda}  \;.
  \end{aligned}\end{equation*}
\end{proof}
Given a target approximation level $\ccH < 1$, we construct a $\ccH$-MIPS as follows.
Define 
\begin{align}\label{eq:def-J}
  J := \lt\lcl \log_{1/\sqrt{\ccH}} \fr{M_{\max}}{M_{\min}} \rt\rcl = O(\log (dT) / \log(\ccH^{-1})) \;.
\end{align}
We build a series of $J$ MIPS hashing schemes that are
\begin{align}\label{eq:mips-series}
  (\ccH^{1/2}, \ccH^{j/2}M_{\max})\text{-MIPS} \text{ for } j \in [J]\;.
\end{align}
We say that the MIPS hashing \emph{succeeds} (\emph{fails}) for a query $\q$ if the retrieved vector is (not) $(c,M)$-MIPS w.r.t. $\q$.
Theorem~\ref{thm:mips} shows that one can perform a binary search to find a vector $\x\in\cX$ that is $\ccH$-MIPS. %that a binary search over the MIPS hashings above suffices to perform a $\ccH$-MIPS hashing.
\begin{thm}\label{thm:mips}
  Upon given a query $\q$, perform a binary search over the $J$ MIPS hashings~\eqref{eq:mips-series} to find the smallest $j^* \in [J]$ for which the retrieved vector $\x^{(j^*)}$ from the $j^*$-th hashing succeeds.
  Then, $\x^{(j^*)}$ is $\ccH$-MIPS w.r.t. $\q$ with high probability.
\end{thm}
\begin{proof}
  %In this proof, we omit the dependence on $R$ for simplicity.
  Assume the event that a retrieved vector $\x^{(j)}$ from $j$-th MIPS satisfies $(\ccH^{1/2},\ccH^{j/2} M_{\max})$-MIPS for $j\in[J]$, which happens with high probability. 

  Define the maximum $M^* := \max_{\x\in\cX_t} \la\q, \x \ra$.
  The result of the binary search is that, for query $\q$, $j^*$-th MIPS succeeds but $(j^*-1)$-th MIPS fails.
  By the definition of $(\ccH^{1/2},\ccH^{(j^*-1)/2} M_{\max})$-MIPS, the fact that $(j^*-1)$-th hashing fails implies $M^* < \ccH^{(j^*-1)/2} M_{\max}$.
  Then,
  \[
    \la\q,\x^{(j^*)}\ra \ge \ccH^{1/2} \cdot \ccH^{j^*} M_{\max} = \ccH \cdot \ccH^{(j^*-1)/2} M_{\max} > \ccH M^*  \;.
  \]
\end{proof} 
Among various $(c,M)$-MIPS algorithms~\cite{shrivastava14asymmetric,shrivastava15improved,neyshabur15on,guo16quantization}, we adopt Shrivastava \emph{et al.}~\cite{shrivastava14asymmetric}. 
Shrivastava et al. propose a reduction of MIPS to locality-sensitive hashing and present a result that their algorithm is $(c,M)$-MIPS with $O(N^{\rho^*}\log N)$ inner product computations and space $O(N^{1+\rho^*})$ for an optimized value $\rho^*$ that is guaranteed to be less than 1; see~\cite[Theorem 5]{shrivastava14asymmetric} for detail.
%We adopt their reduction technique of MIPS to LSH. % algorithm of~\cite{indyk12approximate}.

Let $d'$ be the dimensionality of the projection vectors, where $d'=d^2+d$ for QGLOC and $d'=d$ for GLOC-TS.
One can recover the order of the space and time complexity stated in the main text by $O(\log(J) N^{\rho^*} \log (N) d')$ and $O( J N^{\rho^*} (N + \log (N)d') )$, respectively.

% \paragraph{Space and Time Complexity}
% Here, $k= O(\log N)$ and $U = O(N^{\rho^*}) $, where $\rho^*$ is an optimized value that is always less than 1. % and depends on $\ccH$.
% Although the dependencies between $(\ccH, \dtH)$ and $(J, U, k)$ are complicated and often omitted here, we remark that as we increase $\ccH$ and reduce $\dtH$ (more accurate) we need to increase $J$, $U$, and $k$ (more space and time).

%%%%%%%%%%%%%%%%%%%%%%%%%%%%%%%%%%%%%%%%%%%%%%%%%%%%%%%%%%%%%%%%%%%%%%%%%%%%%%%%
\vspace{-5pt}
\section{The Regret Bound of QGLOC under Hashing Approximation Error}
\vspace{-5pt}
%%%%%%%%%%%%%%%%%%%%%%%%%%%%%%%%%%%%%%%%%%%%%%%%%%%%%%%%%%%%%%%%%%%%%%%%%%%%%%%%
For most recommendation or interactive retrieval applications, it is reasonable to assume that the total number of steps $T$ is bounded by a known constant (``fixed budget'' in bandit terminology) since users do not interact with the system for too long.
%%% BEG: can I support the argument above with some references?
%Indeed, in interactive image retrieval the number of rounds is at most 25.
%%% END
We present the regret of QGLOC combined with MIPS hashing in Theorem~\ref{thm:x_t9_lsh}.
The theorem states that in the fixed budget setting $T$, we can set the target approximation level $\ccH$ as a function of $T$ and enjoy the same order of regret as exactly solving the maximization problem~\eqref{eq:qglocopt}.
The proof is presented at the end of this section.
\begin{thm}\label{thm:x_t9_lsh}
  Let $T\ge 2$ and $\emph{\ccH} = \lt(1 + \fr{\log(T)}{\sqrt{T}}\rt)^{-1}$.
  Suppose we run QGLOC for $T$ iterations where we invoke a $\ccH$-MIPS hashing algorithm $\cH$ to approximately find the solution of~\eqref{eq:qglocopt} at each time $t \in [T]$.
  Denote this algorithm by $\QGLOC\la\cH\ra$.
  Assume that, w.p. at least $1-\dt$, the hashing $\cH$ successfully retrieves a $\ccH$-MIPS solution for every $T$ queries made by $\QGLOC\la\cH\ra$.
%   Assume the probability that given any sequence of $T$ queries to the hashing $\cH$ successfully retrieves a $\ccH$-MIPS solution is at least $1-\dt$. 
  Then, w.p. at least $1-3\dt$, $\emph{\text{Regret}}_T^{\emph{\text{QGLOC}}\la\cH\ra} = \hat O( \kap^{-1}L(L+R) d^{5/4} \sqrt{T} \log^{7/4}(T) ) $.
\end{thm}
Note that we have made an assumption that, w.h.p., the hashing retrieves a $\ccH$-MIPS solution for every $T$ queries made by the algorithm.
One naive way to construct such a hashing is to build $T$ independent hashing schemes as follows, which is the first low-regret GLB algorithm with time sublinear in $N$, to our knowledge.
\begin{cor}\label{cor:x_t9_lsh}
  Let $T\ge 2$ and $\emph{\ccH} = \lt(1 + \fr{\log(T)}{\sqrt{T}}\rt)^{-1}$.
  Suppose we build $T$ independent hashings where each hashing is $\ccH$-MIPS w.p. at least $1-\dt/T$.
  Suppose we run QGLOC for $T$ iterations where at time $t$ we use $t$-th MIPS hashing to solve~\eqref{eq:qglocopt}.
  Then, w.p. at least $1-3\dt$, the regret bound is $\hat O( \kap^{-1}L(L+R) d^{5/4} \sqrt{T} \log^{7/4}(T) )$. 
\end{cor}
\begin{proof}
  The probability that at least one of $T$ hashing schemes fails to output a $\ccH$-MIPS solution is at most $\dt$.
  Combining this with Theorem~\ref{thm:x_t9_lsh} completes the proof.
\end{proof}
%This is the first GLB algorithm that has a sublinear time complexity in $N$, to our knowledge.
However, it is easy to see that its space complexity is $\Omega(T)$, which is not space-efficient.
Of course, a better way would be to construct one hashing scheme that is, w.h.p., $\ccH$-MIPS for any sequence of $T$ queries.
However, this is nontrivial for the following reason.
In bandit algorithms, the second query depends on the result of the first query.
Here, the result of the query $\q_1$ is obtained based on the hash keys computed using the \emph{projection vectors}.
Now, the second query $\q_2$ is based on the result of querying $\q_1$, which means that $\q_2$ now correlates with the projection vectors.
This breaks the \emph{independence} of the query with the projection vectors, thus breaking the hashing guarantee.
One way to get around the issue is the union bound. 
This is possible when there exists a finite set of possible queries, which is indeed how~\cite{indyk12approximate} manage to show a guarantee on such `adaptive' queries.
In our case, unfortunately, there are infinitely many possible queries, and thus the union bound does not apply easily.
Resolving the issue above is of theoretical interest and left as future work.            

Meanwhile, it is hardly the case that the correlation between $\q_2$ and the projection vectors has an malignant effect in practice.
Indeed, many existing studies using hashing such as~\citet{jain10hashing} ignore the correlation issue and do not provide any guarantee on the adaptive queries.

%%%%%%%%%%%%%%%%%%%%%%%%%%%%%%%%%%%%%%%%
\subsection{Proof of Theorem~\ref{thm:x_t9_lsh}}
%%%%%%%%%%%%%%%%%%%%%%%%%%%%%%%%%%%%%%%%

\begin{proof}
  Assume that $\th^*$ belongs to $C^{\ONS}_t$, which happens w.h.p.
  In this proof, we drop $\QGLOC$ from $\x_t^{\QGLOC}$ and use $\x_t$.
  Denote by $\x_t^{\cH}$ the solution returned by the MIPS algorithm.
  We omit $\ONS$ from $\beta^{\ONS}_t$ to avoid clutter.
  Then, being $\ccH$-MIPS guarantees that, $\forall t \in [T]$,
\begin{equation}\begin{aligned} \label{hoful-mips-1} 
    \textstyle    \la \hatth_{t-1}, \x_t^{\cH} \ra + \fr{\beta_{t-1}^{1/4}}{4 c_0 \barm_{t-1}} { ||\x_t^{\cH}||^2_{\barV^{-1}_{t-1}} }
  \ge c_\cH \lt( \la \hatth_{t-1}, \x_t \ra + \fr{\beta_{t-1}^{1/4}}{4 c_0 \barm_{t-1}} { ||\x_t||^2_{\barV^{-1}_{t-1}} } \rt) \;.
\end{aligned}\end{equation}
  To avoid clutter, we use $\X$, $\y$, and $\bfeta$ in place of $\X_{t-1}$, $\y_{t-1}$, and $\bfeta_{t-1} $, respectively, when it is clear from the context.
  Note that, using the techniques in the proof of Theorem~\ref{thm:regret_o2cs},
\begin{equation}\begin{aligned} \label{eq:thm-hoful-mips-a} 
    \la \hatth_{t-1}, \x_t^{\cH} \ra  
    &= \la\th^*,\x_t^{\cH}\ra + \la \hth_{t-1} - \th^*, \x_t^{\cH} \ra \le S + \sqrt{\barbeta_{t-1}}||\x_t^\cH||_{\ibarV_{t-1}} \;.
\end{aligned}\end{equation}
  The instantaneous regret at time $t$ divided by $L$ is
  \begin{align*}
    \fr{r_t}{L} 
  &\le \la \th^*, \x_{t,*} \ra - \la \th^*, \x_t^{\cH} \ra \\
  &\le \la \hatth_{t-1}, \x_t \ra + \fr{\beta_{t-1}^{1/4}}{4 c_0 \barm_{t-1}} { ||\x_t||^2_{\barV^{-1}_{t-1}} } + c_0 \beta_{t-1}^{3/4} \barm_{t-1}  - \la \th^*, \x_t^{\cH} \ra \\
  &\stackrel{\eqref{hoful-mips-1}}{\le} \fr{1}{\ccH} \lt( \la \hatth_{t-1}, \x^{\cH}_t \ra + \fr{\beta_{t-1}^{1/4}}{4 c_0 \barm_{t-1}} { ||\x^\cH_t||^2_{\barV^{-1}_{t-1}} } \rt) + c_0 \beta_{t-1}^{3/4} \barm_{t-1}  - \la \th^*, \x_t^{\cH} \ra \\
  &=  {\lt(\fr{1}{\ccH} - 1\rt) \la \hatth_{t-1}, \x_t^{\cH} \ra} + 
  {\la \hatth_{t-1} - \th^*, \x_t^{\cH} \ra + \fr{1}{\ccH}\fr{\beta_{t-1}^{1/4}}{4 c_0 \barm_{t-1}} { ||\x_t^{\cH}||^2_{\barV^{-1}_{t-1}} } + c_0 \beta_{t-1}^{3/4} \barm_{t-1} } \\
  &\stackrel{\eqref{eq:thm-hoful-mips-a}}{\le} \underbrace{\fr{\log T}{\sqrt{T}} S}_{A_1(t)} + 
    \underbrace{\fr{\log T}{\sqrt{T}} \sqrt{\barbeta_{t-1}} || \x_t^{\cH} ||_{\ibarV_{t-1}} }_{A_2(t)}   + 
    \underbrace{\la \hatth_{t-1} - \th^*, \x_t^{\cH} \ra + \fr{1}{\ccH}\fr{\beta_{t-1}^{1/4}}{4 c_0 \barm_{t-1}} { ||\x_t^{\cH}||^2_{\barV^{-1}_{t-1}} } + c_0 \beta_{t-1}^{3/4} \barm_{t-1} }_{A_3(t)} \;.
  \end{align*}
  Note that $r_t \le 2LS$.
  Since $1/\ccH \le 2$, it is not hard to see that $\sum_{t=1}^T \min\{2LS,L\cdot A_3(t)\}$ is 
    \[O( \kap^{-1}L(L+R) d^{5/4} \sqrt{T} \log^{7/4}(T) )\] 
  using the same technique as the proof of Theorem~\ref{thm:x_t9}. 
  It remains to bound $\sum_{t=1}^T \min\{2LS,L\cdot A_1(t)\}$ and $\sum_{t=1}^T \min\{2S,L\cdot A_2(t)\}$:
  \begin{align*}
    L\sum_{t=1}^T \min\{2S,  A_1(t) \} &\le L\sum_{t=1}^T \min\{2S, \fr{\log T}{\sqrt T} 2S\} \\
    &\stackrel{(\log T < \sqrt{T})}{\le} L\sum_{t=1}^T  \fr{\log T}{\sqrt T} 2S   \\
    &= O(L\sqrt{T} \log T) \\
    L\sum_{t=1}^T \min\{2S, A_2(t) \}
    &= L\sum_{t=1}^T  \min\lt\{ 2S, \fr{\log T}{\sqrt T} \sqrt{\barbeta_{t-1}} || \x_t^{\cH} ||_{\ibarV_{t-1}} \rt\} \\
    &\le L\sum_{t=1}^T  \min\lt\{ 2S, \sqrt{\barbeta_{t-1}} || \x_t^{\cH} ||_{\ibarV_{t-1}} \rt\} \\
  \end{align*}
  Using the same argument as the proof of Theorem~\ref{thm:regret_o2cs} and Corollary~\ref{cor:regret_glocon_ons},
  \begin{align*}
    \sum_{t=1}^T \min\{2S,L\cdot A_2(t)\} = O\lt(\fr{L(L+R)}{\kap}d\sqrt{T}\log^{3/2}(T)\rt)  \;.
  \end{align*}
  Altogether, we notice that $\sum_{t=1}^T \min\{2S,L\cdot A_3(t)\}$ dominates the other terms.
  Notice that we have spent $\dt$ probability to control the event $\th^* \in C^{\ONS}_t$, another $\dt$ for controlling the deviation of $g_t$ which appears in the regret bound through $\beta_t$ as shown in the proof of Theorem~\ref{thm:regret_o2cs}, and another $\dt$ for ensuring the hashing guarantee.
  This sums to $3\dt$ and concludes the proof.
\end{proof}

%%%%%%%%%%%%%%%%%%%%%%%%%%%%%%%%%%%%%%%%%%%%%%%%%%%%%%%%%%%%%%%%%%%%%%%%%%%%%%%%
\vspace{-5pt}
\section{Proof of Lemma~\ref{lem:p_l1}}
\label{sec:proof-lem_p_l1}
\vspace{-5pt}
%%%%%%%%%%%%%%%%%%%%%%%%%%%%%%%%%%%%%%%%%%%%%%%%%%%%%%%%%%%%%%%%%%%%%%%%%%%%%%%%
\begin{proof}
% Definitions:
% \begin{align*}
% \mG_k &= \frac{\vq_{i_k} \va_{i_k}}{p_{i_k}} \\
% p_i &= \frac{\lv \vq_i \rv}{\lV \vq \rV_1}
% \end{align*}
% 
% Computation of variance of $\mG_k$:
% \begin{align*}
% \Var \mG_k &= \sum_{i = 1}^{d} p_i \left( \frac{\vq_{i} \va_{i}}{p_{i}} \right)^2  \\
% &= \sum_{i = 1}^{d}  \frac{\vq_{i}^2 \va_{i}^2}{p_{i}}  \\
% &= \sum_{i = 1}^{d}  \frac{\vq_{i}^2 \va_{i}^2}{\lv \vq_i \rv/\lV \vq \rV_1}  \\
% &= \lV \vq \rV_1 \sum_{i = 1}^{d}  \lv \vq_{i} \rv \va_{i}^2  \\
% &\leq \lV \vq \rV_1 \lV \vq \rV_\infty \sum_{i = 1}^{d}  \va_{i}^2  \\
% &= \lV \vq \rV_1 \lV \vq \rV_\infty \lV \va \rV_2^2  \\
% &\leq \lV \vq \rV_1 \lV \vq \rV_2 \lV \va \rV_2^2  \\
% &\leq \sqrt{d} \lV \vq \rV_2^2 \lV \va \rV_2^2  \\
% \end{align*}
% 
We first recall that for L1 sampling:
$$ p_i = \frac{ \lv q_{i} \rv}{\lV \vq \rV_1} $$

Note that:
\begin{align*}
\lv G_k \rv &= \lv \frac{q_{i_k} a_{i_k}}{p_{i_k}} \rv \\
&= \frac{ \lv q_{i_k} a_{i_k} \rv}{\lv q_{i_k} \rv/ \lV \vq \rV_1 } \\
&=  \lV \vq \rV_1 \lv a_{i_k} \rv \\
&\leq \lV \vq \rV_1 \lV \va \rV_{\max} =: M
\end{align*}

This means that $G_k$ is a bounded random variable. %, and thus sub-Gaussian with scale $M$. 
Therefore, we can use Hoeffding's inequality to get a high-probability bound:
\begin{align*}
X_i &= G_i - \vq^\top \va \\
\Pr \left( \lv \frac{1}{m} \sum_{i=1}^m  X_i \rv \geq \epsilon \right) %&\leq 2 \exp \left(- \frac{2 m^2 \epsilon^2}{4 m M^2} \right) \\ 
&\leq 2 \exp \left(- \frac{ m \epsilon^2}{2 M^2} \right) \\
&\leq 2 \exp \left(- \frac{ m \epsilon^2}{2 \lV \vq \rV_1^2 \lV \va \rV_{\max}^2} \right)
\end{align*}

% \Rightarrow m &\geq \frac{M^2}{\epsilon^2} \log \frac{2}{\delta} \\
% \Rightarrow m &\geq \frac{\lV \vq \rV_1^2 \lV \va \rV_2^2}{\epsilon^2} \log \frac{2}{\delta} \\
\end{proof}

%%%%%%%%%%%%%%%%%%%%%%%%%%%%%%%%%%%%%%%%%%%%%%%%%%%%%%%%%%%%%%%%%%%%%%%%%%%%%%%%
\vspace{-5pt}
\section{Comparison between L1 and L2 sampling} %The Concentration Inequality of L2 Sampling}%Proof of Lemma~\ref{lem:p_l2}}
\label{sec:proof-lem_p_l2}
\vspace{-5pt}
%%%%%%%%%%%%%%%%%%%%%%%%%%%%%%%%%%%%%%%%%%%%%%%%%%%%%%%%%%%%%%%%%%%%%%%%%%%%%%%%

We first show the known high probability error bound of L2, which has a polynomially decaying probability of failure.
\begin{lem} \cite[Lemma 3.4]{jain10hashing} \label{lem:p_l2}
  Define $G_k$ as in~\eqref{eq:def_iprod} with $\p=\p^{(\emph\tL2)}$.
  Then, given a target error $\eps>0$,
  \vspace{-4pt}
  \begin{equation}\begin{aligned}
    \textstyle \P\lt( \lt|\fr{1}{m}\sum_{k=1}^m G_k - \q^\T\a \rt| \ge \eps \rt) \le \fr{||\q||_2^2||\a||_2^2}{m\eps^2} \;.\label{eq:lem-p_l2}
  \end{aligned}\end{equation}
\end{lem}
\vspace{-8pt}

\begin{proof}
This proof is a streamlined version of the proof of Lemma 3.4 from~\cite{jain10hashing}. 
Note that
\[
  \Var G_k \leq \lV \vq \rV^2 \lV \va \rV^2 \;.
\] 
Define
\[
  X_m := \sum_{i=1}^m \left( G_i - \vq^\top \va \right) 
\]
Now, since $\vq^\top \va$ is deterministic,
\[
  \Var X_m = \sum_{i=1}^m \Var G_i \leq m \lV \vq \rV^2 \lV \va \rV^2 
\]
and,
$$ \mathbb{E} X_m = 0 $$

Let us apply Chebyshev's inequality to $X_m$:
\begin{align*}
\Pr \left( \lv X_m - \mathbb{E}X_m \rv \geq \alpha \right) &\leq \frac{\Var X_m}{\alpha^2} \\
\Rightarrow \Pr \left( \lv X_m \rv \geq m \epsilon \right) &\leq \frac{ m \lV \vq \rV^2 \lV \va \rV^2 }{m^2 \epsilon^2} \\
\Rightarrow \Pr \left( \lv \frac{1}{m} \sum_{i=1}^m G_i - \vq^\top \va \rv \geq \epsilon \right) &\leq \frac{ \lV \vq \rV^2 \lV \va \rV^2 }{m \epsilon^2} \\
\Rightarrow \Pr \left( \lv \frac{1}{m} \sum_{i=1}^m G_i - \vq^\top \va \rv \geq \epsilon' \lV \vq \rV \lV \va \rV  \right) &\leq \frac{ \lV \vq \rV^2 \lV \va \rV^2 }{m \epsilon'^2 \lV \vq \rV^2 \lV \va \rV^2 } \mbox{ , where $ \epsilon = \epsilon' \lV \vq \rV \lV \va \rV $} \\
\Rightarrow \Pr \left( \lv \tilde{\vq}^\top \va  - \vq^\top \va \rv \geq \epsilon' \lV \vq \rV \lV \va \rV  \right) &\leq \frac{ 1 }{m \epsilon'^2 } \\
\Rightarrow \Pr \left( \lv \tilde{\vq}^\top \va  - \vq^\top \va \rv \geq \epsilon' \lV \vq \rV \lV \va \rV  \right) &\leq \frac{ 1 }{c} \mbox{ , where $ m = c/\epsilon'^2$} \\
\Rightarrow \Pr \left( \lv \tilde{\vq}^\top \va  - \vq^\top \va \rv \geq \epsilon' \lV \vq \rV \lV \va \rV  \right) &\geq 1 - \frac{ 1 }{c}
\end{align*}
\end{proof}

To compare the concentration of measure of L1 and L2, we look at so-called ``sample complexity'' of L1 and L2.
Let $\dt'$ be the target failure probability (set the RHS of~\eqref{eq:lem-p_l1} to $\dt'$).
Then, the sample complexity of L2 and L1 is $\fr{||\q||^2_2||\a||^2_2}{\dt' \eps^2}$ and $\log(2/\dt')\fr{2||\q||^2_1||a||^2_{\max}}{\eps^2}$, respectively.
Note L2 has a smaller scaling with $\q$ since $||\q||_2 \le ||\q||_1$, but L1 has a better scaling with $\a$ since $||\a||_{\max} \le ||\a||_2$.
More importantly, the concentration bound of L2 decays polynomially with $m$ whereas that of L1 decays exponentially.
However, it is unclear whether or not the polynomial tail of L2 is just an artifact of analysis.
In fact, we find that L2 has an exponential tail bound, but its scaling with $\q$ can be very bad.

While we state the result below in Lemma~\ref{lem:p_l2_exp}, the key here is that the magnitude of $G_k$ induced by L2 is at most $||\q||_2^2 \max_i |a_i|/|q_i|$ while that induced by L1 is at most $||\q||_1 ||\a||_{\max} $.
As $q_i$ goes to zero, the support of the L2-based $G_k$ can be arbitrarily large.
Unless one knows that $q_i$ is sufficiently bounded away from 0 (which is false in QGLOC), L2-based $G_k$ raises a concern of having a ``thick'' tail.\footnote{
  Still, the distribution does not belong to so-called ``heavy-tailed'' distributions.
}
\begin{lem}\label{lem:p_l2_exp}
  Use $\p = \p^{(\text{L2})}$. Then,
  \begin{equation*}\begin{aligned}
    \P \left( \lv \frac{1}{m} \sum_{i=1}^m  X_i \rv \geq \epsilon \right) \leq 2 \exp \left(- \frac{ m \epsilon^2}{2 \lV \vq \rV_2^4 \max_i \lv a_i/q_i \rv^2 } \right) \\
  \end{aligned}\end{equation*}
\end{lem}
\begin{proof}
From the results in Lemma 3.4 of~\cite{jain10hashing} (restated as Lemma~\ref{lem:p_l2} in our paper), it seems that L2 sampling has polynomial tails, which is considered as heavy-tailed.
Here we show that the tail of L2 sampling does not decay polynomially but exponentially by using Hoeffding's inequality instead of Chebychev's inequality.
However, the scale that controls the tail thickness is quite bad, and the tail can be arbitrarily thick regardless of the norm or variance of $\q$.

We first recall that for L2 sampling:
$$ p_i = \frac{q_i^2}{\lV \vq \rV_2^2} $$

Let us first show an upper bound for $\lv G_k \rv$:
\begin{align*}
\lv G_k \rv &= \lv \frac{q_{i_k} a_{i_k}}{p_{i_k}} \rv \\
&= \frac{ \lv q_{i_k} a_{i_k} \rv}{ q_{i_k}^2/ \lV \vq \rV_2^2 } \\
&= \lV \vq \rV_2^2 \lv \frac{a_{i_k}}{q_{i_k}} \rv \\
&\leq \lV \vq \rV_2^2 \max_i \lv \frac{a_{i}}{q_{i}} \rv =: M' \\
\end{align*}

Recall that $\mathbb{E} G_k = 0$ and from above, $\lv G_k \rv \leq M'$. We can now apply Hoeffding's inequlity:
\begin{align*}
X_i &= G_i - \vq^\top \va \\
\Pr \left( \lv \frac{1}{m} \sum_{i=1}^m  X_i \rv \geq \epsilon \right) %&\leq 2 \exp \left(- \frac{2 m^2 \epsilon^2}{4 m M'^2} \right) \\ 
&\leq 2 \exp \left(- \frac{ m \epsilon^2}{2 {M'}^2} \right) \\
&\leq 2 \exp \left(- \frac{ m \epsilon^2}{2 \lV \vq \rV_2^4 \max_i \lv\frac{a_{i}}{q_{i}} \rv^2 } \right) \\
\end{align*}
This is an exponential tail bound compared to the polynomial tail bound in Lemma 3.4 in~\cite{jain10hashing}. We do note that the term $\max_i \lv\frac{a_{i}}{q_{i}} \rv$ could be very bad if $\min_i q_i$ is small and the corresponding $a_i$ is non-zero.
\end{proof}  

The second comparison is on the variance of L1 and L2.
The variance of $G_k$ based on $\p^{(\tL2)}$ can be shown to be  $||\q||_2^2||\a||_2^2 - (\q^\T \a)^2$.
With L1, the variance of $G_k$ is now $||\q||_1 (\sum_{i=1}^d |q_i| a_i^2) - (\q^\T \a)^2$ whose first term can be upper-bounded by $||\q||_1 ||\q||_2 ||\a||_4^2 $ using Cauchy-Schwartz.
This means that the variance induced by L1 scales larger with $\q$ than by L2 since $||\q||_2 \le ||\q||_1$ and scales smaller with $\a$ since $||\a||_4^2 \le ||\a||_2^2$. 
Thus, neither is an absolute winner.
However, if the vectors being inner-producted are normally distributed, then L1 has a smaller variance than L2 in most cases, for large enough $d$ as we show in the lemma below.
As mentioned in the main text, our projection vectors are truly normally distributed.

\begin{lem}\label{lem:l1-var}
Suppose that $\vq \sim \mathcal{N}(0, \I_d)$ and $\va \sim \mathcal{N}(0, \I_d)$ and that $\vq$ and $\va$ are independent of each other. Then $\mathbb{E} \left( \lV \vq \rV_1 \sum_{i=1}^d \lv q_i \rv a_i^2 \right) \leq \mathbb{E} \left( \lV \vq \rV_2^2 \lV \a \rV_2^2 \right)$ for large enough $d$.
\label{lem:lem_var_gaussian}
\end{lem}

\begin{proof}

First note that if $\x \sim \mathcal{N}(0, I_d)$, then:
\begin{align*} 
\mathbb{E} \lV \x \rV_2^2 &= d \;.
\end{align*}

Then, the RHS is
\begin{align*}
\mathbb{E} \lV \vq \rV_2^2 \lV \a \rV_2^2  &= \mathbb{E} \lV \vq \rV_2^2 \mathbb{E} \lV \a \rV_2^2 \\
&= d^2 \;.
\end{align*}

For the LHS:
\begin{align*}
\mathbb{E} \left( \lV \vq \rV_1 \sum_{i=1}^d \lv q_i \rv a_i^2 \right) &= \mathbb{E}_\vq \left( \lV \vq \rV_1 \sum_{i=1}^d \lv q_i \rv \mathbb{E}_\va a_i^2 \right) \\
&= \mathbb{E}_\vq \left( \lV \vq \rV_1 \sum_{i=1}^d \lv q_i \rv \right) \\
&= \mathbb{E}_\vq  \lV \vq \rV_1^2 \\
&= \mathbb{E}_\vq \left( \sum_{i=1}^d \lv q_i \rv^2 + 2 \sum_{i,j=1, i \neq j}^d \lv q_i \rv \lv q_j \rv \right) \\
&= \mathbb{E}_\vq \sum_{i=1}^d  \lv q_i \rv^2 + 2 \sum_{i,j=1, i \neq j}^d \mathbb{E}_\vq \lv q_i \rv \lv q_j \rv \\
&= d + 2 \frac{2}{\pi} \frac{d(d-1)}{2} \\
&= d + \frac{2}{\pi} d(d-1) \\
&= \frac{2}{\pi} d^2 + \frac{\pi - 2}{\pi} d  \;.
\end{align*}
where we used the fact that $\sum_{i=1}^d  \lv q_i \rv^2$ follows a Chi-squared distribution with $d$ degrees of freedom and that $\lv q_i \rv$ follows a half-normal distribution, whose mean is $\sqrt{2/\pi}$.

Since $2/\pi < 1$, for large $d$, $ \frac{2}{\pi} d^2 + \frac{\pi - 2}{\pi} d  < d^2$.
\end{proof}

%%%%%%%%%%%%%%%%%%%%%%%%%%%%%%%%%%%%%%%%%%%%%%%%%%%%%%%%%%%%%%%%%%%%%%%%%%%%%%%%
\section{Hashing Implementation}
%%%%%%%%%%%%%%%%%%%%%%%%%%%%%%%%%%%%%%%%%%%%%%%%%%%%%%%%%%%%%%%%%%%%%%%%%%%%%%%%

We implement hashing for QGLOC and GLOC-TS based on python package OptimalLSH\footnote{\url{https://github.com/yahoo/Optimal-LSH}}.
In practice, building a series of hashings with varying parameter shown in Section~\ref{sec:supp-hashing} can be burdensome.
Departing from the theory, we use so-called multi-probe technique~\cite{slaney12optimal} that achieves a similar effect by probing nearby buckets.
That is, upon determining a hash key and retrieving its corresponding bucket, one can also look at other hash keys that are different on only one component.
We use $k=12$ keys and $U=24$ tables.

\end{document}